\documentclass{article}

\PassOptionsToPackage{numbers, compress}{natbib}

\usepackage[main,final]{neurips_2026}

\usepackage[utf8]{inputenc} 
\usepackage[T1]{fontenc}    
\usepackage{hyperref}       
\usepackage{url}            
\usepackage{booktabs}       
\usepackage{amsfonts}       
\usepackage{nicefrac}       
\usepackage{microtype}      
\usepackage{xcolor}         

\usepackage{amsmath,amsfonts,bm}

\def\eqref#1{equation~\ref{#1}}

\def\1{\bm{1}}

\DeclareMathAlphabet{\mathsfit}{\encodingdefault}{\sfdefault}{m}{sl}
\SetMathAlphabet{\mathsfit}{bold}{\encodingdefault}{\sfdefault}{bx}{n}

\def\gC{{\mathcal{C}}}

\def\gU{{\mathcal{U}}}
\def\gV{{\mathcal{V}}}
\def\gW{{\mathcal{W}}}

\newcommand{\KL}{D_{\mathrm{KL}}}

\DeclareMathOperator*{\argmax}{arg\,max}

\usepackage{xspace}
\usepackage{soul}
\usepackage{url}
\usepackage{graphicx}
\usepackage{amsmath}
\usepackage{amsfonts}
\usepackage{multirow}
\usepackage{amsthm}
\usepackage{booktabs}
\usepackage{longtable}
\usepackage{epstopdf}
\usepackage{colortbl}
\usepackage[labelformat=simple]{subcaption}
\usepackage{caption}
\usepackage{bbm}
\usepackage{dsfont}
\usepackage[ruled,linesnumbered]{algorithm2e}
\usepackage{setspace}
\usepackage{braket}
\usepackage{tablefootnote}
\usepackage{subcaption}
\usepackage[normalem]{ulem}
\usepackage{hyperref}
\usepackage{cleveref}
\usepackage{url}
\usepackage{makecell}
\usepackage{wrapfig}
\usepackage{sansmath}
\usepackage{float}
\usepackage[table,xcdraw]{xcolor}
\usepackage{pifont}
\useunder{\uline}{\ul}{}
\theoremstyle{plain}
\newtheorem{proposition}{Proposition}[section]

\usepackage{tcolorbox}
\usepackage[toc,page,header]{appendix}
\usepackage{minitoc}
\usepackage{listings}

\newcommand{\ourmethod}{MATO\xspace}
\newcommand{\method}{\textsf{MATO}\xspace}
\graphicspath{{images/}}

\definecolor{carnelian}{rgb}{0.7, 0.11, 0.11}

\title{\method: Multi-objective Personalized Alignment with Test-time Optimization for Large Language Models}

\author{\textbf{Linhao Luo}\textsuperscript{1}
        \textbf{Thuy-Trang Vu}\textsuperscript{1}       
        \textbf{Van-Anh Nguyen}\textsuperscript{1} 
        \textbf{Junae Kim}\textsuperscript{2} \\
        \textbf{Gholamreza Haffari}\textsuperscript{1} 
        \textbf{Dinh Phung}\textsuperscript{1} \\
        \textsuperscript{1}Monash University \quad \textsuperscript{2} Defence Science and Technology Group, Australia \\
        \texttt{\{linhao.luo,trang.vu1,van-anh.nguyen\}@monash.edu} \\
        \texttt{junae.kim@defence.gov.au} \\
        \texttt{\{gholamreza.haffari,dinh.phung\}@monash.edu} \\
        }

\begin{document}
\doparttoc 
\faketableofcontents 


\maketitle

\begin{abstract}
Aligning large language models (LLMs) with diverse and multifaceted user preferences is a fundamental challenge in personalized AI systems. Existing multi-objective alignment methods either rely on costly training or require pre-trained reward models for each preference, making it difficult for them to adapt to evolving preferences. Prompt-based personalization offers a training-free alternative, but prompting alone often provides limited steerability, as LLMs may overemphasize or overlook certain preferences and fail to give users reliable control over the relative importance of different objectives when conflicts arise, leading to suboptimal alignment.
In this paper, we introduce \method, a training-free framework for \textbf{M}ulti-objective personalized \textbf{A}lignment with \textbf{T}est-time \textbf{O}ptimization. \method formulates personalization as a test-time optimization problem that steers the relative importance of multiple objectives through controllable weights during decoding, without modifying model parameters or requiring external reward models. Specifically, a \emph{reward discovery} module constructs prompt-derived surrogate rewards from the backbone LLM for diverse objectives specified in natural language, while a \emph{weight optimization} module dynamically adjusts objective weights based on the user's initial preferences and the partially generated response to balance competing objectives during generation. The resulting surrogate rewards and weights jointly guide an \emph{online optimization} procedure over the token distribution, enabling better alignment with the target objectives.
Extensive experiments across multiple datasets and backbone LLMs show that \method consistently outperforms strong baselines, achieving Pareto-improving multi-objective alignment and stronger steerability. These results highlight test-time optimization as a promising direction for scalable, controllable, and model-agnostic personalized alignment. Code is available at \url{https://github.com/rmanluo/MATO}.
\end{abstract}

\section{Introduction}\label{sec:introduction}

Large language models (LLMs) have emerged as general-purpose assistants that support a wide range of user needs, including writing~\citep{wasi2024llms}, teaching~\citep{meyer2023chatgpt}, coding~\citep{jiang2026survey}, and planning~\citep{dong2023towards}. Their growing applications make it increasingly important to align their behavior with user preferences and values~\citep{li2024personal}. Mainstream alignment methods, such as reinforcement learning from human feedback (RLHF)~\citep{ouyang2022training}, typically learn from feedback aggregated across users, which tends to collapse diverse preferences into a single reward signal, making them less suited to personalized alignment~\citep{poddar2024personalizing}.
In practice, user preferences are inherently diverse and multifaceted, shaped by differences in background, goals, and values~\citep{chen2024large,zhangpersonalization}. This challenge has motivated growing interest in multi-objective personalized alignment~\citep{zhou2024beyond,yang2024rewards} that aligns LLMs with diverse user preferences, involving multiple and potentially competing objectives.

Early research on multi-objective personalized alignment extends RLHF to optimize LLMs on linearly combined multi-preference rewards~\citep{li2020deep,rame2023rewarded,zhou2024beyond}, which requires extensive training of separate models for different preference combinations. More recent work has explored test-time alignment methods~\citep{lin2025parm,chenpad} that adapt LLMs to multiple objectives without additional LLM training. However, these methods still require pre-trained reward models for each preference, which can be hard to obtain as user preferences evolve and new objectives emerge~\citep{guan2025survey}. 

Prompt-based personalization is appealing because it is training-free and can adapt on the fly across LLMs and objectives~\citep{zhufly,hecontext}.
Recent instruction-following LLMs can often adapt their outputs to preference descriptions expressed in natural language in the prompt~\citep{lee2024aligning,linunlocking}. However, prompting alone is often insufficient for robust multi-objective alignment~\citep{son2025robust}. As shown in \Cref{fig:intro}, when users specify multiple objectives, LLMs may over-prioritize some preferences (e.g., helpfulness), neglect others (e.g., friendliness), or behave inconsistently when objectives conflict (e.g., informativeness vs.\ conciseness). More fundamentally, prompt conditioning offers limited steerability~\citep{gao2024linear,shi2024decoding}, as LLMs may fail to infer the relative importance of objectives and balance them reliably, leading to suboptimal alignment and poor trade-offs among preferences~\citep{zhaollms}. 

\begin{figure}
    \centering
    \includegraphics[width=0.8\textwidth]{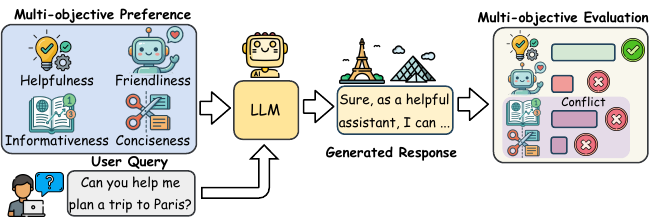}
    \caption{Illustration of multi-objective personalized alignment and its challenges. 
    }
    \label{fig:intro}
\end{figure}

In this paper, we introduce \method, a training-free framework for \textbf{M}ulti-objective personalized \textbf{A}lignment via \textbf{T}est-time \textbf{O}ptimization. Rather than encoding preference control in model parameters or relying solely on natural-language prompts, \method formulates personalization as a test-time optimization problem, allowing users to steer the relative importance of multiple objectives through controllable weights during LLM decoding. To remove the need for external reward models, we propose a \emph{reward discovery} module that constructs prompt-derived surrogate rewards from the backbone LLM for objectives described in natural language. We further introduce a \emph{weight optimization} module that dynamically adjusts objective weights based on the user's initial preferences and the partially generated response to balance preferences during generation. The surrogate rewards and weights are used to guide an \emph{online optimization} procedure over the model's output distribution to achieve better alignment with the target objectives. This design better handles competing preferences, exposes a steerable multi-objective trade-off, and enables Pareto-improving alignment. Extensive experiments across multiple datasets and backbone LLMs show that \method consistently outperforms strong baselines in overall alignment quality while providing stronger steerability. These results highlight test-time optimization as a promising direction for scalable, controllable, and model-agnostic personalized alignment.

\section{Related Work}\label{sec:relatedwork}

\noindent\textbf{Training-based personalized alignment.}
Prior work formulates personalized alignment as multi-objective optimization by extending RLHF with linearly combined rewards (MORL)~\citep{li2020deep}. Preference optimization methods such as MODPO~\citep{zhou2024beyond} improve stability and efficiency over MORL, while approaches like Rewarded Soups~\citep{rame2023rewarded} and RiC~\citep{yang2024rewards} reduce the need for training separate models by interpolating weights or conditioning on preference vectors.
%
Recently, test-time alignment methods keep the backbone LLM fixed and guide decoding with reward models~\citep{chenpad,xugenarm}. These include logit-level merging across objectives (MOD)~\citep{liu2024decoding}, unified preference-aware reward modeling (PARM)~\citep{lin2025parm}, and robust decoding strategies (RMOD)~\citep{son2025robust}. Despite improved efficiency, these approaches still rely on pre-trained reward models for each objective, limiting adaptability as preferences evolve or new objectives emerge~\citep{guan2025survey}.

\noindent\textbf{Training-free personalized alignment.}
The training-free paradigm has gained increasing attention for personalized alignment due to its flexibility and efficiency. 
Prompt-based methods exploit instruction-following to express user preferences in natural language~\citep{lee2024aligning}, while decoding-based techniques steer generation via preference signals, including alignment directions (Linear Alignment)~\citep{gao2024linear}, interpolation of prompt-conditioned outputs (CoS)~\citep{hecontext}, divergence-based rewards (OPAD)~\citep{zhufly}, and online token distribution updates (Amulet)~\citep{zhangamulet}.
%
%
Despite their flexibility and efficiency, these methods are primarily designed for single-objective alignment and struggle to manage trade-offs among multiple preferences. In contrast, our \method~formulates training-free personalized alignment as a multi-objective optimization problem at test time, enabling balanced and adaptive preference control.
Detailed comparisons with these methods are provided in \Cref{tab:related_comparison}.
\section{Preliminary}\label{sec:preliminary}
\noindent\textbf{Multi-objective Personalized Alignment.} 
We study explicit personalized alignment, where the target objectives and their relative importance are supplied as inputs to a language model $\pi$. The implicit personalization~\citep{zhaollms} where preferences must first be inferred, is outside our scope.
%
Given a query $x$, the model generates $y \sim \pi(\cdot \mid x)$. Let $\mathcal{C}={\{c_k\}}_{k=1}^K$ denote natural-language objective descriptions, e.g., ``use concise language'' or ``explain concepts for a beginner,'' with corresponding reward functions $R_k$ and non-negative weights $w_k$ ($\sum_k w_k=1$). The supplied initial weights are denoted by $\gW^{\mathrm{init}}$ in \Cref{sec:weight_optimization}. The goal is to maximize the weighted reward:
\begin{equation}
    y^* = \argmax_{y \sim \pi(\cdot \mid x)} \quad \sum_{k=1}^K w_k R_k(x, y).\label{eq:multi_objective}
\end{equation}
This formulation captures the trade-offs among preferences, while the weights allow users to control their relative importance and target a personalized response.

\noindent\textbf{Test-time Alignment.} 
Based on the objective defined in \Cref{eq:multi_objective}, a direct approach is to apply multi-objective reinforcement learning (MORL) by optimizing the LLM on a weighted combination of preference rewards~\citep{li2020deep,rame2023rewarded,zhou2024beyond}, formulated as: 
\begin{equation}
    \pi^* = \arg \max_\pi \mathbb{E}_{y \sim \pi(\cdot \mid x)} \left[ \sum_{k=1}^K w_k R_k(x, y) \right] - \beta \KL(\pi(\cdot \mid x) || \pi_{\text{base}}(\cdot \mid x)).\label{eq:morlhf}
\end{equation}
This requires training separate models for different preference combinations. Test-time methods~\citep{xugenarm,lin2025parm} instead guide the decoding of a fixed LLM $\pi_{\text{base}}$ with token-level reward models $R_k(y_i|x, y_{<i})$ for each preference as:
\begin{equation}
    \pi^*(y_i \mid x, y_{<i}) \propto \pi_{\text{base}}(y_i \mid x, y_{<i})
    \exp\left(\frac{1}{\beta}\sum_{k=1}^K w_k R_k(y_i \mid x, y_{<i})\right),\label{eq:test_time}
\end{equation}
where $y_i$ is the $i$-th token in the generated response.
Despite their effectiveness, these methods still require pre-trained token-level reward models for each preference, which are costly to obtain as user preferences evolve and new objectives emerge~\citep{guan2025survey}.


\section{Approach}\label{sec:approach}
\begin{figure}
    \centering
    \includegraphics[width=1\textwidth]{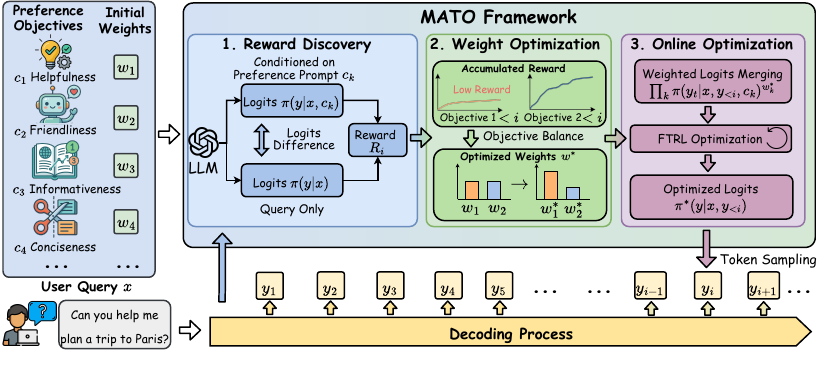}
    \caption{Overview of the \method framework, which consists of three key components: (1) \emph{reward discovery}, which constructs prompt-derived surrogate rewards from the backbone LLM, (2) \emph{weight optimization}, which dynamically adjusts objective weights to balance preferences during generation, and (3) \emph{online optimization}, which guides the output distribution using the surrogate rewards and optimized weights.}
    \label{fig:method}
\end{figure}

In this section, we introduce \method, a training-free framework for multi-objective personalized alignment with test-time optimization. As shown in \Cref{fig:method}, \method consists of three key components: (1) \emph{reward discovery}, (2) \emph{weight optimization}, and (3) \emph{online optimization}, which we detail below.

\subsection{Reward Discovery}\label{sec:reward_discovery}


Test-time methods rely on pre-trained reward models $R_k$ for each preference, which limits their flexibility
~\citep{guan2025survey}. To address this, a \emph{reward discovery} module is introduced to construct prompt-derived surrogate rewards directly from the backbone LLM, enabling decoding without additional training.
%
%
The closed-form solution of the KL-regularized RL objective in \Cref{eq:morlhf} motivates a token-level reward representation~\citep{rafailov2023direct,rafailov2024from}:
\begin{equation}
    R_k(y_i|x, y_{<i}) := \beta \log \frac{\pi_k^*(y_i \mid x, y_{<i})}{\pi_{\text{base}}(y_i \mid x, y_{<i})},\label{eq:reward_derivation}
\end{equation}
where $\pi_k^*$ is the optimal policy for the $k$-th objective under KL regularization. Each term measures how consistent token $y_i$ is with the $k$-th preference under the optimal policy, and they could be aggregated into a sequence-level reward (\Cref{sec:proof_reward_derivation}).

Since $\pi_k^*$ is unavailable, we use the backbone conditioned on a preference-specific prompt $c_k$ as a proxy~\citep{zhufly,hecontext} to approximate the optimal policy and obtain the \emph{prompt-derived surrogate reward}:
\begin{equation}
    \widehat R_k(y_i|x, y_{<i}) := \beta \log \frac{\pi_{\text{base}}(y_i \mid x, y_{<i}, c_k)}{\pi_{\text{base}}(y_i \mid x, y_{<i})},\label{eq:reward_discovery}
\end{equation}
where $c_k$ describes the $k$-th preference in natural language. This signal measures the quality of preference alignment without additional reward model training. 
Experiments in \Cref{sec:appendix_surrogate_validation} show that the surrogate reward is positively associated with alignment quality, supporting its use as a reward proxy.
The prompt templates are given in \Cref{sec:appendix_reward_prompt}.

\subsection{Weight Optimization}\label{sec:weight_optimization}
In multi-objective alignment, some preferences may dominate while others are neglected during generation, leading to suboptimal alignment performance~\citep{zhaollms,son2025robust}.
Thus, we aim to balance the preferences by increasing the weights of neglected preferences and decreasing the weights of dominant ones while maximizing the overall rewards, thereby ensuring Pareto-improving alignment during generation. 
To this end, we track the accumulated surrogate reward for each preference over the generated prefix, which provides a running proxy for objective satisfaction and enables dynamic weight adjustment.
Specifically, we leverage the reward discovery module in \Cref{eq:reward_discovery} to compute accumulated surrogate rewards by summing token-level surrogate rewards on generated tokens as
\begin{equation}
    \widehat R_k(x, y_{<i}) = \sum_{j=1}^{i-1} \widehat R_k(y_j|x, y_{<j}),\label{eq:cumulative_reward}
\end{equation}
which can be efficiently computed in parallel during decoding. Impact of other reward aggregation methods is studied in \Cref{sec:appendix_reward_aggregation}.

Drawing on the adversarial weight-player view of max-min MORL~\citep{park2024max,byeonmulti}, we treat the preference weights as a simplex-constrained adversary that reweights the cumulative surrogate rewards at the current decoding prefix as:
\begin{gather}
    \max_{\pi} \min_{\gW \in \Delta^K} \sum_{k=1}^K w_k \widehat R_k(x, y_{<i}).\label{eq:max_min_morl}
\end{gather}
However, directly solving the resulting max-min problem at every decoding step is computationally expensive and typically requires iterative optimization over the simplex~\citep{son2025robust,huang2024one}. To preserve user steerability while avoiding per-step iterative optimization, we introduce a personalized regularized weight subproblem anchored at the user's initial preference weights $\gW^{\mathrm{init}}=\{w_1^{\mathrm{init}},\ldots,w_K^{\mathrm{init}}\}\in\Delta^K$:
\begin{gather}
    \gW^*(x, y_{<i}) = \arg \min_{\gW \in \Delta^K} \sum_{k=1}^K w_k \widehat R_k(x, y_{<i}) + \tau \KL(\gW || \gW^{\mathrm{init}}),\label{eq:weight_optimization}
\end{gather}
where $\Delta^K$ is the probability simplex in $\mathbb{R}^K$, and $\tau$ controls how far the optimized weights can move away from the user prior $\gW^{\mathrm{init}}$.

\begin{proposition}
    For any generated prefix $y_{<i}$, if $\tau>0$ and $w_k^{\mathrm{init}}>0$ for all $k$, the regularized weight subproblem in \Cref{eq:weight_optimization} admits the unique closed-form solution
\begin{gather}
    w_k^* = \frac{w_{k}^{\mathrm{init}} \cdot \exp\left(-\frac{\widehat R_k(x, y_{<i})}{\tau}\right)}{\sum_{j=1}^K w_j^{\mathrm{init}} \exp\left(-\frac{\widehat R_j(x, y_{<i})}{\tau}\right)},\label{eq:closed_form_weight}
\end{gather} 
\end{proposition}
A higher $\tau$ preserves more of the user prior, while a lower $\tau$ enables more aggressive reweighting toward under-optimized objectives. The detailed proof is provided in \Cref{sec:proof_weight_derivation}.

\subsection{Online Optimization}\label{sec:online_optimization}

Injecting the surrogate rewards (\cref{eq:reward_discovery}) and optimized weights (\cref{eq:closed_form_weight}) into \Cref{eq:test_time} yields a weighted decoding distribution:
$\prod_{k=1}^K \pi_{\text{base}}(y_i \mid x, y_{<i}, c_k)^{w^*_k}$, we further multiply this term by a fixed multi-objective prompt-conditioned distribution as a stabilizing anchor, obtaining:
\begin{gather}
\begin{split}
    \pi'(y_i \mid x, y_{<i},\gC) \propto \pi_{\text{base}}(y_i \mid x, y_{<i},\gC) \cdot 
    \prod_{k=1}^K \pi_{\text{base}}(y_i \mid x, y_{<i}, c_k)^{w^*_k},\label{eq:fused_logits}
\end{split}
\end{gather}
where $\gC$ combines all preference descriptions and the user's initial weights into a fixed instruction prompt (\Cref{sec:appendix_alignment_prompt}). This anchor keeps the merged distribution close to a coherent multi-objective instruction-following policy and improves decoding stability.
The merged distribution $\pi'$ combines the objective-specific surrogate signals under the optimized weights. We further apply online optimization to amplify this combined signal under KL regularization to move towards better alignment with the multi-objective preferences.

Specifically, we use the log-probability ratio between the current iterate and the base distribution as a token-level utility signal~\citep{zhufly}, with KL penalties toward the merged distribution $\pi'$ and the previous iterate. These updates use only model-derived signals to push generation toward high-quality regions defined by $\gC$, indirectly approaching optimal policy $\pi^*$~\citep{zhufly} (see \Cref{sec:proof_divergence}). 
Starting from $\pi_0 = \pi'$, we perform $T$ iterative updates using follow-the-regularized-leader (FTRL)~\citep{mcmahan2011follow}:
\begin{gather}
    \pi_{t} = \arg \max_\pi \sum_{j=1}^{t-1} \gU_j(\pi) - \frac{1}{\eta}  \KL(\pi || \pi_{t-1}),\ \pi_0 = \pi'\label{eq:ftrl},\\
    \gU_j(\pi) = \sum_{a\in\mathcal{V}} \pi(a) u_j(a) - \lambda \KL(\pi || \pi'),\label{eq:utility}\\
    u_j(a) = \alpha \left(\log \pi_j(a|x, y_{<i}) - \log \pi_{\text{base}}(a|x, y_{<i})\right), \quad a\in\mathcal{V},\label{eq:reward}
\end{gather}
where $a\in\gV$ denotes a token from the vocabulary $\gV$. The utility function $\gU$ (\cref{eq:utility}) encourages the updated policy $\pi_t$ to maximize divergence from the base distribution $\pi_{\text{base}}$ (\cref{eq:reward}) while regularizing it to stay close to the merged distribution $\pi'$. The FTRL update (\cref{eq:ftrl}) ensures stability via a KL regularization term that penalizes large deviations from the previous policy $\pi_{t-1}$.

The optimization procedure in \Cref{eq:ftrl} can be efficiently solved in closed-form~\citep{zhangamulet} as
\begin{gather}
    \pi_t(a) \propto \exp \left(\frac{1}{(t-1)\lambda\eta +1}\left(\eta\sum_{j=1}^{t-1} u_j(a) + (t-1)\lambda\eta \log\pi'(a) + \log \pi_{t-1}(a) \right) \right), \quad a\in\mathcal{V}.\label{eq:closed_form_ftrl}
\end{gather}
A detailed derivation can be found in \Cref{sec:proof_online_optimization}. 
After $T$ steps, we sample $y_i$ from $\pi_T$ to generate responses better aligned with multi-objective preferences, which is empirically assessed in \Cref{sec:ablation_study}. The full decoding process is in \Cref{alg:mato_decoding}.

\section{Experiment}\label{sec:experiment}
In our experiments, we aim to answer the following research questions: (1) How does \method perform in multi-objective personalized alignment (\Cref{sec:alignment_performance})? (2) How does \method enable steerability of the multi-objective trade-off (\Cref{sec:steerability})?
(3) How does \method dynamically adjust weights to balance objectives during decoding (\Cref{sec:dynamics})? (4) How do different components and parameters of \method contribute to its performance (\Cref{sec:ablation_study})? (5) How efficient is \method in terms of computational cost and latency (\Cref{sec:efficiency})?  
\begin{wrapfigure}{r}{0.60\textwidth}
\centering
\resizebox{0.56\textwidth}{!}{%
\begin{minipage}{0.56\textwidth}
\begin{algorithm}[H]
\DontPrintSemicolon
\caption{LLM Decoding Process with \method}
\label{alg:mato_decoding}
\KwIn{LLM $\pi_{\text{base}}$; query $x$; preference prompts $\gC=\{c_1,\ldots,c_K\}$; initial weights $\gW^{\mathrm{init}}=\{w_1^{\mathrm{init}},\ldots,w_K^{\mathrm{init}}\}$; weight optimization temperature $\tau$, online optimization steps $T$; hyperparameters $\alpha$, $\lambda$, $\eta$}
\KwOut{Response $y$}
$y \leftarrow \emptyset$\;
\While{generation has not ended}{
\tcp{Reward Discovery.}
Generate $\pi_{\text{base}}(\cdot \mid x, y)$, $\pi_{\text{base}}(\cdot \mid x, y,\gC)$, and $\pi_{\text{base}}(\cdot \mid x, y, c_k)$, $\forall k \in \{1,\ldots,K\}$, with LLM\;
\For{$k = 1, 2, \ldots, K$}{
Compute the accumulated surrogate reward $\widehat R_k(x, y)$ by \Cref{eq:reward_discovery,eq:cumulative_reward}\;
}
\tcp{Weight Optimization.}
Update the optimized weights $\gW^*$ by \Cref{eq:closed_form_weight}\;
Construct the merged distribution $\pi'$ by \Cref{eq:fused_logits}\;
\tcp{Online Optimization.}
$\pi_0 \gets \pi'$\;
\For{$t = 1, 2, \ldots, T$}{
Update $\pi_t$ with the FTRL objective in \Cref{eq:ftrl}\;
}
Sample the next token $y_i \sim \pi_T(\cdot \mid x, y)$\;
Append the sampled token to the response $y \leftarrow y \oplus y_i$\;
}
\Return the generated response $y$
\end{algorithm}
\end{minipage}%
}
\vspace{-4em}
\end{wrapfigure}

\subsection{Experimental Setup}\label{sec:experimental_setup}
\noindent\textbf{Dataset.} For multi-objective personalized alignment, we evaluate on the \emph{Multifaceted} dataset~\citep{lee2024aligning}, which defines 107 personalized user preferences across 4 dimensions and 18 subdimensions. It contains 945 user queries, each randomly assigned 4 preferences with uniform weights. The full list of preferences is provided in \Cref{tab:total_preferences}.
For steerability evaluation, we use the \emph{HH-RLHF} dataset~\citep{bai2022training}, which has been widely used in previous multi-objective alignment work~\citep{yang2024rewards,shi2024decoding}. We select the humor and helpfulness preferences to evaluate the steerability of multi-objective trade-off. The weights of the two preferences are varied from 1 to 0 with a step size of 0.2 to evaluate the model's ability to balance the two objectives under different weight combinations.

\noindent\textbf{Evaluation Metrics.} For multi-objective personalized alignment, we use GPT-4o and claude-sonnet-4-5 as the LLM-as-judge to score each response on each preference dimension ($s_k \in [1, 5]$) based on the rubrics provided by the Multifaceted dataset~\citep{lee2024aligning}. To quantify the overall multi-objective alignment performance, we compute three metrics: (1) \emph{all preference match rate} (AMR), which measures the percentage of responses that fully match all preferences (i.e., $\prod_{k=1}^{K} \mathbb{I}(s_k \geq 3) = 1$); (2) \emph{average preference score} (APS), which measures the average score across all preferences; and (3) \emph{worst preference score} (Worst), which measures the lowest score among all preferences. The prompt for the LLM-as-judge and detailed metric definitions are provided in \Cref{sec:appendix_evaluation}.
For steerability evaluation, we use two pre-trained reward models to provide continuous reward scores for responses on helpfulness\footnote{\url{https://huggingface.co/Ray2333/gpt2-large-helpful-reward_model}} and humor\footnote{\url{https://huggingface.co/mohameddhiab/humor-no-humor}} preferences, which allow us to plot the Pareto front of the model's performance under different weight combinations. These reward models have been widely used in previous work~\citep{yang2024rewards,shi2024decoding}.

\noindent\textbf{Baselines.} We compare \method with two groups of baselines: \emph{training-free methods} and \emph{training-based methods}.

Training-free methods require no gradient-based training and align LLMs with text-based prompts, including: \emph{Base LLM}, which directly evaluates the base LLM on the questions without any alignment; \emph{Preference Prompt}~\citep{lee2024aligning}, which concatenates multiple preference descriptions into prompts for the base LLM; \emph{Linear Alignment}~\citep{gao2024linear}, which performs test-time personalized alignment with a linear update on the token distribution; \emph{CoS}~\citep{hecontext}, which amplifies the influence of the preference prompt by comparing two model outputs; \emph{Amulet}~\citep{zhangamulet}, which conducts online learning on the distribution to align with a single preference and \emph{OPAD}~\citep{zhufly}, which uses a surrogate preference-guided reward model to optimize the output distribution to align with human preferences. 

Training-based methods require training on the target preferences, including: \emph{MORLHF}~\citep{li2020deep}, which uses reinforcement learning to train separate LLMs with different weight combinations on the rewards of the target preferences; \emph{RiC}~\citep{yang2024rewards}, which conducts supervised fine-tuning on the base LLM by concatenating multiple preferences and their weights as input prompts; and \emph{MOD}~\citep{liu2024decoding}, which trains several single-objective aligned LLMs and performs weighted decoding to balance the preferences.  Because these methods require objective-specific training, they are not practically applicable to the Multifaceted dataset~\citep{lee2024aligning}, which involves a large number of preferences.

\noindent\textbf{Implementation Details.} For training-free methods, we use four different base LLMs of varying sizes: \texttt{Qwen3-0.6B}, \texttt{Qwen3-8B}, \texttt{Llama-3.2-1B}, and \texttt{Llama-3.1-8B}. As most training-free methods are not designed for multi-objective personalized alignment, we adapt them to the multi-objective setting by concatenating multiple preference descriptions into a single prompt (\Cref{sec:appendix_alignment_prompt}).
For training-based methods, we use \texttt{Qwen3-8B} as the base LLM. We conduct training on the training set of the HH-RLHF dataset using the aforementioned pre-trained reward models. We evaluate the trained models on the test set of the HH-RLHF dataset following the same evaluation protocol as previous work~\citep{yang2024rewards,shi2024decoding}. For \method, we use the same base LLMs as in the training-free methods. We set the weight optimization temperature $\tau$ to 1, the number of online optimization steps $T$ to 80, and $\alpha=0.5,\lambda=1.0,\eta=10$.

\subsection{Multi-objective Personalized Alignment Performance}\label{sec:alignment_performance}

\begin{table}[]
\centering
\caption{Performance comparison on multi-objective personalized alignment using GPT-4o as judge. 
}
\label{tab:alignment_overall}
\resizebox{1\textwidth}{!}{%
\begin{tabular}{@{}l|ccc|ccc|ccc|ccc@{}}
\toprule
\multirow{2}{*}{Methods} & \multicolumn{3}{c|}{Qwen3-0.6B} & \multicolumn{3}{c|}{Qwen3-8B} & \multicolumn{3}{c|}{Llama-3.2-1B} & \multicolumn{3}{c}{Llama-3.1-8B} \\
                         & AMR  & APS  & Worst & AMR  & APS  & Worst & AMR  & APS  & Worst & AMR  & APS  & Worst \\ \midrule
Base LLM                 & 0.14 & 2.58 & 1.49  & 0.43 & 3.66 & 2.34  & 0.21 & 2.78 & 1.69  & 0.34 & 3.36 & 2.04  \\
Preference Prompt~\citep{lee2024aligning}        & 0.19 & 2.78 & 1.64  & 0.70 & 4.19 & 3.23  & 0.31 & 3.07 & 2.00  & 0.66 & 3.97 & 2.96  \\
Linear Alignment~\citep{gao2024linear}         & 0.24 & 2.80 & 1.76  & 0.72 & 4.19 & 3.31  & 0.37 & 2.78 & 2.11  & 0.64 & 3.87 & 2.98  \\
CoS~\citep{hecontext}                      & 0.19 & 2.59 & 1.64  & 0.66 & 4.00 & 3.13  & 0.17 & 2.34 & 1.60  & 0.52 & 3.52 & 2.63  \\
Amulet~\citep{zhangamulet}                   & 0.12 & 2.31 & 1.46  & 0.60 & 3.79 & 2.94  & 0.12 & 2.06 & 1.42  & 0.32 & 2.92 & 2.07  \\
OPAD~\citep{zhufly}                     & 0.25 & 2.97 & 1.80  & 0.74 & 4.28 & 3.39  & 0.36 & 2.99 & 1.88  & 0.68 & 4.03 & 3.10  \\\midrule
\method                  & \textbf{0.27} & \textbf{2.99} & \textbf{1.87} & \textbf{0.75} & \textbf{4.30} & \textbf{3.42} & \textbf{0.39} & \textbf{3.23} & \textbf{2.25} & \textbf{0.73} & \textbf{4.08} & \textbf{3.22} \\ \bottomrule
\end{tabular}%
}
\end{table}

\noindent\textbf{Overall Performance.} 
Table~\ref{tab:alignment_overall} compares multi-objective personalized alignment performance. \method consistently outperforms all baselines across base LLMs and metrics, demonstrating its effectiveness and generalizability in aligning LLMs with multiple personalized preferences.
Preference Prompt substantially improves over the Base LLM, showing that in-context learning from explicit preference descriptions is already an effective alignment signal, particularly for larger LLMs. 
 However, Linear Alignment, CoS, and OPAD, which further optimize the token distribution based on the preference prompts during decoding, perform closely to Preference Prompt (e.g., identical APS of 4.19 on Qwen3-8B).
  This reveals that treating multiple preferences as a single combined signal cannot track or compensate for individual preference satisfaction.
   Moreover, their performance is even worse than Preference Prompt on smaller LLMs (i.e., Llama-3.2-1B, APS of 2.78 vs. 3.07), which may be attributed to the incorrect guidance from the combined signal, which misleads the optimization direction. 
   Despite employing online learning, Amulet performs worst overall, as its single-objective optimization direction conflicts with the multi-objective setting and leads to misaligned generation.

%
\begin{figure}[t]
    \centering
    \begin{minipage}[t]{0.5\textwidth}
        \centering
        \includegraphics[width=\textwidth,clip,trim=0.2cm 0.3cm 0.3cm 0.3cm]{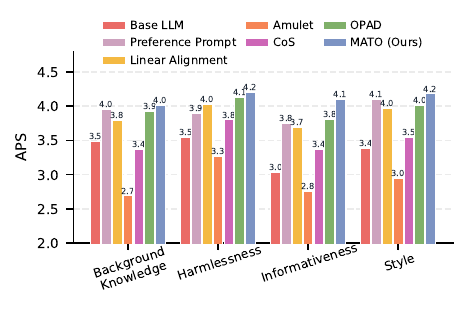}
        \caption{Average preference score across dimensions.}
        \label{fig:bar_aps_by_dimension}
    \end{minipage}\hfill
    \begin{minipage}[t]{0.4\textwidth}
        \centering
        \includegraphics[width=\textwidth]{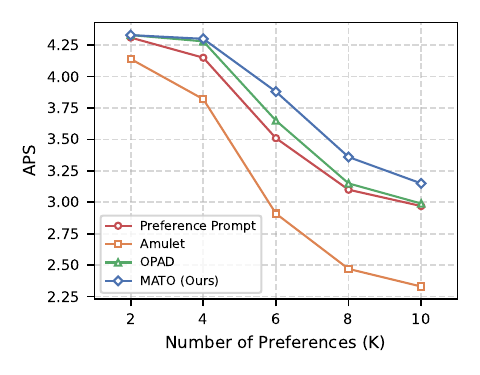}
        \caption{Average preference score under different aligned preferences.}
        \label{fig:scale_by_preference}
    \end{minipage}
    \vspace{-1.5em}
\end{figure}
In contrast, \method constructs per-preference surrogate rewards and dynamically reweights underserved objectives, with online optimization amplifying the merged preference signal under KL regularization. The advantage is most pronounced on the Worst metric, which measures the minimum score across all preferences and directly reflects multi-objective balance. For instance, on Llama-3.2-1B, \method outperforms OPAD by $+0.37$ on Worst and $+0.24$ on APS, confirming that \method effectively prevents any single preference from being neglected and thus achieves better overall multi-objective alignment.
To ensure the robustness of our results, we also evaluate the multi-objective personalized alignment performance with an additional LLM judge (claude-sonnet-4-5) in \Cref{sec:appendix_additional_llm_as_judge}, which shows consistent relative performance among different methods and further confirms the effectiveness of \method.
Experiments on larger base LLMs (Qwen3-32B) are provided in \Cref{sec:appendix_32b}. Performance of original tasks are provided in \Cref{sec:appendix_helpfulness}. 
Case studies are provided in \Cref{sec:appendix_case_study}. 

\noindent\textbf{Performance by Dimension.} To further demonstrate the effectiveness of \method in balancing multiple preferences, we report the performance on each preference dimension. As shown in Figure~\ref{fig:bar_aps_by_dimension}, existing alignment methods tend to have imbalanced performance across different dimensions, while \method achieves more balanced performance across all dimensions with APS consistently above 4.0. This further confirms the advantage of \method in multi-objective alignment and preventing any single preference from being neglected.

\noindent\textbf{Scaling to Different Numbers of Objectives.} To evaluate the scalability of \method to different numbers of objectives, we augment the original 4-objective setting by randomly sampling different numbers of preferences (from 2 to 10) for each query and evaluating the overall performance of different methods. As shown in Figure~\ref{fig:scale_by_preference}, while the performance of all methods generally decreases as the number of preferences increases, \method consistently outperforms all baselines across different numbers of preferences, demonstrating its scalability in multi-objective alignment.

The latency analysis under different aligned preferences $K$ is provided in \Cref{sec:appendix_latency}. To increase alignment efficiency, we further implement a \emph{parallel decoding} where the objective-conditioned branches can be batched and executed in parallel. This approach reduces the measured latency to 0.12--0.15 seconds per token, while retaining the objective-wise signals used for dynamic balancing. These results highlight the benefit of parallel decoding when aligning a larger number of objectives.


\subsection{Multi-objective Pareto Steerability}\label{sec:steerability}

To evaluate steerability, we vary the helpfulness-humor weight from 1 to 0 with a step size of 0.2, compute the reward for each objective, and plot the empirical Pareto front~\citep{rame2023rewarded}. For training-based methods, we train separate models following their original algorithms. For training-free baselines, we directly inject each weight combination into the preference prompt at inference time (\Cref{sec:appendix_alignment_prompt}).

\begin{wrapfigure}{r}{0.45\textwidth}
    \centering
    \vspace{-0.4\baselineskip}
    \includegraphics[width=0.45\textwidth,clip,trim=0.4cm 0.4cm 0.4cm 0.4cm]{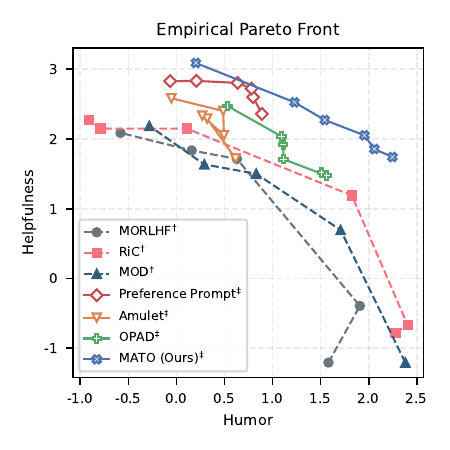}
    \vspace{-0.5cm}
    \caption{
    Empirical Pareto fronts of different methods:
    $\dagger$ and $\ddagger$ denote training-based and training-free methods, respectively.}
    \label{fig:pareto_front}
    \vspace{-0.5cm}
\end{wrapfigure}

Figure~\ref{fig:pareto_front} shows that training-based methods (MORLHF, RiC, and MOD) produce smooth and well-distributed Pareto fronts, indicating good steerability under different weights. However, their overall alignment performance is generally lower than that of training-free methods, likely due to train-test distribution shift. They are also more vulnerable to exploit reward model weaknesses without producing genuinely meaningful responses. Examples are provided in \Cref{sec:appendix_steerability_case_study}.

Training-free methods such as Preference Prompt and OPAD achieve strong alignment but ill-distributed Pareto fronts, indicating limited steerability. A possible reason is that they collapse multiple preferences into a single prompt, making it difficult to capture the relative importance of each objective.
In contrast,
\method explicitly models each preference and fuses logits according to target weights (\Cref{eq:fused_logits}), while the dynamic weight optimization (\Cref{eq:closed_form_weight}) adapts to the current generation state. This yields a substantially better Pareto front than all training-free baselines, achieving both strong alignment and effective steerability.


\subsection{Reward and Weight Dynamics}\label{sec:dynamics}
\begin{figure}
    \centering
    \includegraphics[width=1\textwidth,clip,trim=0.2cm 0.4cm 0.2cm 0.2cm]{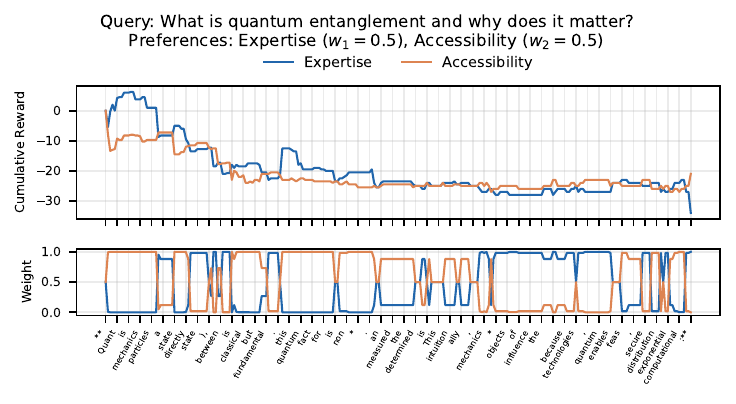}
    \caption{Illustration of the surrogate reward and weight dynamics of each preference during \method decoding.}
    \label{fig:dynamics}
    \vspace{-1.5em}
\end{figure}

To understand how \method dynamically adjusts reward weights during decoding, we visualize the surrogate reward and weight dynamics of two competing preferences (Expertise vs. Accessibility). As shown in Figure~\ref{fig:dynamics}, the surrogate rewards for different preferences fluctuate throughout decoding, and \method adjusts the weights accordingly to compensate for underserved preferences. As a result, the surrogate rewards of the two preferences remain close and stay aligned with the initial user weights (0.5 vs. 0.5), demonstrating the effectiveness of \method in dynamically balancing multiple preferences during generation. More case studies with different initial weights are provided in \Cref{sec:appendix_dynamics_case_study}.

\subsection{Ablation Study and Parameter Analysis}\label{sec:ablation_study}

The ablation study and parameter analysis are conducted with \texttt{Qwen3-0.6B} as the base LLM on a subset of the Multifaceted dataset for efficiency.

\begin{wraptable}{r}{0.5\columnwidth}
\centering
\caption{Ablation study of \method.}
\label{tab:ablation_study}
\resizebox{0.5\columnwidth}{!}{%
\begin{tabular}{@{}lccc@{}}
\toprule
Methods & AMR & APS & Worst \\ \midrule
\method                           & \textbf{0.27} & \textbf{2.99} & \textbf{1.87} \\
\method w/o weight optimization   & 0.26          & 2.88          & 1.80          \\
\method w/o online optimization & 0.24         & 2.83          & 1.83          \\
\method w/o both                  & 0.21          & 2.73          & 1.62          \\ \bottomrule
\end{tabular}%
}
\vspace{-1\baselineskip}
\end{wraptable}

\noindent\textbf{Ablation Study.} We conduct ablation studies to evaluate the contribution of each component in \method, including (1) \emph{w/o weight optimization} that removes the weight optimization and directly uses the initial user weights for fusion, (2) \emph{w/o online optimization} that removes the online optimization and directly fuses the logits with the optimized weights, and (3) \emph{w/o both} that removes both modules, which is equivalent to directly fusing the logits from the base LLM with the initial user weights.

Table~\ref{tab:ablation_study} shows that both weight optimization and online optimization contribute to the performance of \method. Removing either degrades all metrics, while removing both causes the largest drop, especially on Worst, confirming that weight optimization rebalances underserved objectives, while online optimization further improves the generation distribution.

Analysis of the reliability of the prompt-based reward discovery is provided in \Cref{sec:appendix_surrogate_validation}, which shows that the surrogate rewards are positively associated with the preference scores judged by LLMs, confirming that the reward discovery module can effectively construct surrogate rewards for multi-objective alignment. Robustness of prompt variants and effects of different reward accumulation are provided in \Cref{sec:appendix_prompt_sensitivity} and \Cref{sec:appendix_reward_aggregation}, respectively. Results show that \method is robust to different prompt variants while the reward accumulation method can affect the surrogate reward quality and thus requiring further investigation.


\begin{wrapfigure}{r}{0.5\textwidth}
    \centering
    \vspace{-1\baselineskip}
    \includegraphics[width=0.5\textwidth,clip,trim=0.2cm 0.3cm 0.2cm 0.2cm]{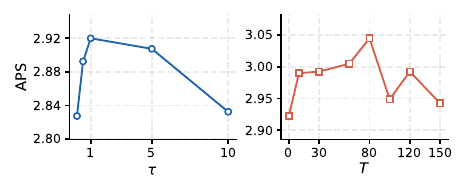}
    \caption{Analysis of weight optimization temperature $\tau$ and 
    optimization steps $T$.}
    \label{fig:parameters}
    \vspace{-1\baselineskip}
\end{wrapfigure}

\noindent\textbf{Parameter Analysis.} We analyze the sensitivity of \method to two key parameters: the weight optimization temperature $\tau$ and the number of online optimization steps $T$. As shown in Figure~\ref{fig:parameters}, the best performance is achieved when $\tau=1$ and $T=80$. A too-small $\tau$ leads to over-sharpened weights that may cause optimization instability, while a too-large $\tau$ results in over-smoothed weights that cannot effectively rebalance the preferences. For the number of online optimization steps, a moderate number of steps allows optimization toward better multi-objective alignment, while too many steps may lead to a distorted optimization direction and poor efficiency. Analysis of other parameters are provided in \Cref{sec:appendix_generation_temperature,sec:appendix_parameter_analysis}.

\subsection{Efficiency Analysis}\label{sec:efficiency}
\begin{wraptable}{r}{0.35\textwidth}
\vspace{-3em}
\centering
\caption{Efficiency comparison.}
\label{tab:efficiency}
\resizebox{0.35\columnwidth}{!}{%
\begin{tabular}{@{}lcc@{}}
\toprule
Methods & \shortstack{Training \\ (hours)} & \shortstack{Inference \\ (seconds / token)} \\ \midrule
MORLHF & 143.77 & 0.05 \\
RiC & 3.4 & 0.05 \\
MOD & 131.43 & 0.12 \\
\midrule
Amulet & 0 & 0.15 \\
OPAD & 0 & 0.10 \\
\method & 0 & 0.16 \\ \bottomrule
\end{tabular}%
}
\vspace{-1\baselineskip}
\end{wraptable}

Table~\ref{tab:efficiency} compares efficiency across training-based (MORLHF, RiC, MOD) and training-free methods (Amulet, OPAD). Training-based methods require substantial optimization time, while \method avoids training entirely. Compared with Amulet and OPAD, \method introduces moderate per-token latency from test-time optimization, effectively shifting cost from training to inference while achieving better alignment. 
The latency analysis under different optimization steps $T$ and aligned preferences $K$ is provided in \Cref{sec:appendix_latency}. 

\section{Conclusion}\label{sec:conclusion}
In this paper, we presented \method, a training-free framework that formulates multi-objective personalized alignment as test-time optimization over the decoding distribution. \method constructs per-objective prompt-derived surrogate rewards from the backbone LLM, dynamically rebalances objective weights around the user's preference in closed form, and refines the token distribution via online optimization, removing the need for per-preference reward models or model training while preserving user steerability. Experiments across four backbone LLMs show consistent gains over training-free and training-based baselines, with the largest improvements on the worst-served objective and substantially better Pareto fronts.


\newpage
\section*{Acknowledgments}
This work was supported by the Australian Defence Science and Technology (DST) Group through the Next Generation Technology Fund (NGTF) scheme and the Department of Defence, Australia, via the Advanced Strategic Capabilities Accelerator (ASCA) program. Dinh Phung further acknowledged the support from the Australian Research Council (ARC) Discovery Project DP230101176 and DP250100262.
\section*{Ethics Statement}
Our research addresses only scientific questions and involves no human subjects, animals, or environmentally sensitive materials. Therefore, we anticipate no ethical risks or conflicts of interest. We are committed to upholding the highest standards of scientific integrity and ethics to ensure the validity and reliability of our findings.

\section*{Reproducibility Statement}
Our model is clearly formalized in the main text for clarity and thorough understanding. Detailed implementation, including dataset information, baselines, experimental settings, and model configurations, are provided in \Cref{sec:experimental_setup,sec:appendix_dataset,sec:appendix_implementation}. Experimental settings and baselines have been rigorously verified to ensure fair comparison. Code is available at: \url{https://github.com/RManLuo/MATO}.

\section*{Usage of LLMs}
LLMs are used to proofread and polish the writing of this paper. We have carefully reviewed and verified all content generated by LLMs to ensure accuracy and integrity. Any errors or inaccuracies in the final manuscript are solely our responsibility.
\bibliographystyle{plainnat}
\bibliography{sections/main}

\appendix
\addcontentsline{toc}{section}{Appendix}
\part{Appendix}
\parttoc
\section{Detailed Related Work Comparison}
\label{sec:appendix_related_work}

\subsection{Training-based Personalized Alignment}
Early research adopts the multi-objective reinforcement learning (MORL) framework~\citep{li2020deep}, which extends RLHF to optimize LLMs on linearly combined multi-preference rewards. MODPO~\citep{zhou2024beyond} extends DPO~\citep{rafailov2023direct} to multi-objective alignment, outperforming MORLHF with greater stability and lower compute cost. However, these methods require training separate models for different preference combinations. Rewarded Soups~\citep{rame2023rewarded} alleviates this issue by interpolating the LLM weights to approximate Pareto fronts. Panacea~\citep{zhong2024panacea} introduces a preference-controlled LoRA adapter that injects a preference vector into the LoRA weights and achieves efficient multi-objective alignment with a single model.
Instead of RL training, RiC~\citep{yang2024rewards} further improves the efficiency by training a single model with preference weights as input using supervised fine-tuning (SFT). Aligner~\citep{ji2024aligner} is fine-tuned on a preference dataset to learn the correctional residuals between preferred and non-preferred responses and then stacked on the upstream model to achieve corrected alignment.
MetaAligner~\citep{yang2024metaaligner} further trains a generalized rewriter that rewrites model outputs based on any user preferences.

Recently, test-time alignment methods~\citep{chenpad,zhang2026memfly,xugenarm} keep the backbone LLM fixed and guide the decoding process with reward models for each preference. MOD~\citep{liu2024decoding} trains multiple LLMs for different objectives and merges their logits during decoding. PARM~\citep{lin2025parm} introduces a preference-aware reward model that can unify multiple preferences and trade-offs. PAD~\citep{chenpad} trains a preference-conditioned reward model that can adapt to evolving user preferences and guide LLM decoding.
RMOD~\citep{son2025robust} further optimizes the decoding process by improving the worst-case reward. Nevertheless, these methods still require pre-trained reward models for each preference, which can be hard to obtain as user preferences evolve and new objectives emerge~\citep{guan2025survey}.

\subsection{Training-free personalized alignment.}
The training-free paradigm has gained increasing attention for personalized alignment due to its flexibility and efficiency. Prompt-based personalization~\citep{lee2024aligning} leverages the instruction-following capabilities of LLMs to adapt outputs to user preferences expressed in natural language. To enhance the alignment quality, Linear Alignment~\citep{gao2024linear} derives a preference alignment direction by comparing model outputs with and without the preference prompt, and steers the decoding process accordingly. Similarly, CoS~\citep{hecontext} amplifies the influence of the preference prompt by comparing two model outputs and controls the level of personalization by adjusting the interpolation weight. OPAD~\citep{zhufly} proposes a preference-guided decoding alignment method guided by a sequential divergence reward design. Amulet~\citep{zhangamulet} adopts an online learning framework to optimize the output token distribution and achieve better alignment with user preferences. However, these methods are designed for single-objective alignment and cannot handle the trade-offs between multiple preferences. In contrast, our \method~introduces a multi-objective optimization framework for training-free personalized alignment, which can effectively balance multiple preferences and adapt to evolving user needs.
Detailed comparisons with these methods are provided in \Cref{tab:related_comparison}.

\begin{table}[]
    \centering
    \caption{Comparisons between previous alignment methods and \ourmethod.}
    \label{tab:related_comparison}
    \resizebox{\columnwidth}{!}{%
        \begin{tabular}{@{}c|l|cccc@{}}
            \toprule
            Category                              & Method        & Training-free & Reward-free & Pareto Steerable & Generalizability \\ \midrule
                                                  & MORLHF~\citep{li2020deep}        & \ding{55}             & \ding{55}           & \checkmark       & \ding{55}                \\
                                                  & MODPO~\citep{zhou2024beyond}         & \ding{55}             & \ding{55}           & \checkmark       & \ding{55}                \\
                                                  & Rewarded Soup~\citep{rame2023rewarded} & \ding{55}             & \ding{55}           & \checkmark       & \ding{55}                \\
            \multirow{-4}{*}{RL Training}         & Panacea~\citep{zhong2024panacea}       & \ding{55}             & \ding{55}           & \checkmark       & \ding{55}                \\ \midrule
                                                  & RiC~\citep{yang2024rewards}           & \ding{55}             & \ding{55}           & \checkmark       & \ding{55}                \\
                                                  & Aligner~\citep{ji2024aligner}       & \ding{55}             & \checkmark  & \ding{55}                & \ding{55}                \\
            \multirow{-3}{*}{SFT}                 & MetaAligner~\citep{yang2024metaaligner}   & \ding{55}             & \checkmark  & \ding{55}                & \checkmark       \\ \midrule
                                                  & MOD~\citep{liu2024decoding}           & \ding{55}             & \ding{55}           & \checkmark       & \ding{55}                \\
                                                  & PARM~\citep{lin2025parm}          & \ding{55}             & \ding{55}           & \checkmark       & \ding{55}                \\
                                                  & PAD~\citep{chenpad}           & \ding{55}             & \ding{55}           & \ding{55}                & \checkmark       \\
                                                  & RMOD~\citep{son2025robust}          & \ding{55}             & \ding{55}           & \ding{55}                & \ding{55}                \\
                                                  & LA~\citep{gao2024linear}            & \checkmark    & \checkmark  & \ding{55}                & \checkmark       \\
                                                  & CoS~\citep{hecontext}           & \checkmark    & \checkmark  & \ding{55}                & \checkmark       \\
                                                  & Amulet~\citep{zhangamulet}        & \checkmark    & \checkmark  & \ding{55}                & \checkmark       \\
                                                  & OPAD~\citep{zhufly}          & \checkmark    & \checkmark  & \ding{55}                & \checkmark       \\ \cmidrule(l){2-6}
            \multirow{-9}{*}{Test-time Alignment} & \method (Ours)      & \checkmark    & \checkmark  & \checkmark       & \checkmark       \\ \bottomrule
        \end{tabular}%
    }
\end{table}

\section{Detailed Theoretical Proofs}
\label{sec:appendix_proofs}

\subsection{Ideal Reward Representation and Prompt-derived Surrogate}\label{sec:proof_reward_derivation}

\begin{proposition}
For the KL-regularized multi-objective RL objective in \Cref{eq:morlhf}, let $\pi_k^*$ denote the optimal policy obtained by optimizing only the $k$-th objective, i.e., with weight vector equal to the $k$-th simplex vertex. Assuming the policies share support and the partition function is finite, a token-level representation of the sequence reward is defined by
\begin{equation}
    R_k(y_i|x, y_{<i}) := \beta \log \frac{\pi_k^*(y_i \mid x, y_{<i})}{\pi_{\text{base}}(y_i \mid x, y_{<i})},
\end{equation}
The aggregation of these terms equals $R_k(x,y)-\beta\log Z_k(x)$, establishing the representation in \Cref{eq:reward_derivation}.
\end{proposition}

\begin{proof}
The proof follows the closed-form solution of KL-regularized RL used in DPO~\citep{rafailov2023direct,rafailov2024from}, adapted to the multi-objective setting considered in this paper.

For a fixed input $x$, the multi-objective RL objective in \Cref{eq:morlhf} is
\begin{equation}
    \pi^*(\cdot \mid x)= \arg \max_\pi \ \mathbb{E}_{y \sim \pi(\cdot \mid x)} \left[ \sum_{k=1}^K w_k R_k(x, y) \right] - \beta \KL(\pi(\cdot \mid x)\|\pi_{\text{base}}(\cdot \mid x)).
\end{equation}
By the standard derivation of KL-regularized RL, the optimizer has the closed form
\begin{equation}
    \pi^*(y \mid x)=\frac{1}{Z(x)}\pi_{\text{base}}(y \mid x)\exp\left(\frac{1}{\beta}\sum_{k=1}^K w_k R_k(x, y)\right),\label{eq:appendix_multi_closed_form}
\end{equation}
where
\begin{equation}
    Z(x)=\sum_y \pi_{\text{base}}(y \mid x)\exp\left(\frac{1}{\beta}\sum_{k=1}^K w_k R_k(x, y)\right)
\end{equation}
is the normalizing partition function. Taking logarithms on both sides of \Cref{eq:appendix_multi_closed_form} gives
\begin{equation}
    \sum_{k=1}^K w_k R_k(x, y)=\beta \log \frac{\pi^*(y \mid x)}{\pi_{\text{base}}(y \mid x)}+\beta \log Z(x).\label{eq:appendix_aggregate_reward}
\end{equation}
Thus, in the multi-objective setting, the weighted aggregate reward is determined by the log-ratio between the optimal policy and the base policy, up to an additive term depending only on $x$.

However, \method requires each objective-specific reward separately. To isolate the $k$-th reward, we consider the special case where the weight vector is the $k$-th simplex vertex, namely $w_k=1$ and $w_j=0$ for all $j \neq k$. Let $\pi_k^*$ denote the optimal policy of this single-objective subproblem:
\begin{equation}
    \pi_k^*(\cdot \mid x)=\arg \max_\pi \ \mathbb{E}_{y \sim \pi(\cdot \mid x)}[R_k(x, y)]-\beta \KL(\pi(\cdot \mid x)\|\pi_{\text{base}}(\cdot \mid x)).
\end{equation}
Applying the same closed-form argument yields
\begin{equation}
    \pi_k^*(y \mid x)=\frac{1}{Z_k(x)}\pi_{\text{base}}(y \mid x)\exp\left(\frac{1}{\beta}R_k(x, y)\right),\label{eq:appendix_single_closed_form}
\end{equation}
where $Z_k(x)$ is the corresponding normalization term. Rearranging \Cref{eq:appendix_single_closed_form}, we obtain
\begin{equation}
    R_k(x, y)=\beta \log \frac{\pi_k^*(y \mid x)}{\pi_{\text{base}}(y \mid x)}+\beta \log Z_k(x).\label{eq:appendix_sequence_reward_k}
\end{equation}

Next, using the autoregressive factorization of language models~\citep{rafailov2024from},
\begin{equation}
    \pi_k^*(y \mid x)=\prod_{i=1}^{|y|}\pi_k^*(y_i \mid x, y_{<i}), \quad
    \pi_{\text{base}}(y \mid x)=\prod_{i=1}^{|y|}\pi_{\text{base}}(y_i \mid x, y_{<i}),
\end{equation}
the log-ratio in \Cref{eq:appendix_sequence_reward_k} decomposes as
\begin{equation}
    \log \frac{\pi_k^*(y \mid x)}{\pi_{\text{base}}(y \mid x)}
    =\sum_{i=1}^{|y|}\log \frac{\pi_k^*(y_i \mid x, y_{<i})}{\pi_{\text{base}}(y_i \mid x, y_{<i})}.
\end{equation}
Substituting this decomposition into \Cref{eq:appendix_sequence_reward_k} shows that, up to an input-dependent additive constant, the sequence-level reward can be written as a sum of token-level rewards:
\begin{equation}
    R_k(x, y)=\sum_{i=1}^{|y|} R_k(y_i \mid x, y_{<i}) + \beta \log Z_k(x),
\end{equation}
where we define
\begin{equation}
    R_k(y_i|x, y_{<i}) := \beta \log \frac{\pi_k^*(y_i \mid x, y_{<i})}{\pi_{\text{base}}(y_i \mid x, y_{<i})}.
\end{equation}
This establishes the chosen representation in \Cref{eq:reward_derivation}. The constant does not affect the optimal policy for fixed objective weights, but objective-specific offsets can affect adaptive comparisons across objectives.
\end{proof}

\paragraph{Prompt-derived surrogate.}
The ideal policy $\pi_k^*$ is unavailable in MATO. Equation~\ref{eq:reward_discovery} instead defines $\widehat R_k$ using the prompt-conditioned backbone. This is an exact definition of the surrogate, not an exact recovery of $R_k$. Its usefulness depends on the backbone's response to $c_k$; the derivation does not establish accuracy relative to ground-truth preferences or calibration across objectives. The weight updates below operate on these surrogate rewards with the log-ratio baseline fixed by \Cref{eq:reward_discovery}.

\subsection{Derivation of Closed-form Weight Optimization}\label{sec:proof_weight_derivation}

\begin{proposition}
Fix a decoding prefix $y_{<i}$ and a regularization strength $\tau>0$, and let the user prior be $\gW^{\mathrm{init}}=\{w_1^{\mathrm{init}},\ldots,w_K^{\mathrm{init}}\}\in\Delta^K$ with $w_k^{\mathrm{init}} > 0$ for all $k$ and $\sum_{k=1}^K w_k^{\mathrm{init}} = 1$. Consider the regularized weight optimization problem in \Cref{eq:weight_optimization}:
\begin{equation}
    \gW^*(x, y_{<i}) = \arg \min_{\gW \in \Delta^K} \sum_{k=1}^K w_k \widehat R_k(x, y_{<i}) + \tau \KL(\gW || \gW^{\mathrm{init}}).
\end{equation}
Then the problem admits a unique optimizer, and its $k$-th component is
\begin{equation}
    w_k^* = \frac{w_{k}^{\mathrm{init}} \cdot \exp\left(-\frac{\widehat R_k(x, y_{<i})}{\tau}\right)}{\sum_{j=1}^K w_{j}^{\mathrm{init}} \cdot \exp\left(-\frac{\widehat R_j(x, y_{<i})}{\tau}\right)},
\end{equation}
which is exactly \Cref{eq:closed_form_weight}.
\end{proposition}

\begin{proof}
We first place the proposed weight optimization into the broader max-min optimization literature~\citep{park2024max}. In max-min MORL, the problem can be viewed as a min-max game in which an adversarial weight player selects a point on the simplex to emphasize the worst-performing objective~\citep{byeonmulti}, namely
\begin{equation}
    \max_{\pi} \min_{\gW \in \Delta^K} \sum_{k=1}^K w_k R_k(x, y_{<i}) = \min_{\gW \in \Delta^K} \max_{\pi}  \sum_{k=1}^K w_k R_k(x, y_{<i}).\label{eq:appendix_max_min_morl}
\end{equation}

The inner maximization over policies corresponds to the KL-regularized multi-objective RL objective in \Cref{eq:morlhf}. As shown in \Cref{eq:appendix_multi_closed_form}, this inner problem admits a closed-form optimal policy for any fixed weight vector $\gW$~\citep{son2025robust}. Motivated by this adversarial reweighting perspective, \method optimizes the weights at the current decoding prefix using the cumulative surrogate rewards, while introducing a KL tether to the user prior $\gW^{\mathrm{init}}$ so that the reweighting remains personalized and admits a closed-form solution. This yields the following regularized weight optimization problem:
\begin{align}
    \gW^*(x, y_{<i}) = & \arg \min_{\gW \in \Delta^K} \sum_{k=1}^K w_k \widehat R_k(x, y_{<i}) + \tau \KL(\gW || \gW^{\mathrm{init}}),\\
    & s.t. \sum_{k=1}^K w_k = 1, \quad w_k \ge 0.
\end{align}

For a fixed prefix $y_{<i}$, define the objective
\begin{equation}
    F(\gW) = \sum_{k=1}^K w_k \widehat R_k(x, y_{<i}) + \tau \KL(\gW || \gW^{\mathrm{init}})
    = \sum_{k=1}^K w_k \widehat R_k(x, y_{<i}) + \tau \sum_{k=1}^K w_k \log \frac{w_k}{w_k^{\mathrm{init}}},
\end{equation}
over the feasible set $\Delta^K$. The first term is linear in $\gW$, whereas the KL term is strictly convex in $\gW$ whenever $\tau > 0$ and $w_k^{\mathrm{init}} > 0$ for all $k$. It follows that $F$ is strictly convex on the simplex, and hence the optimization problem admits a unique minimizer.

To characterize this minimizer, consider the Lagrangian
\begin{equation}
    \mathcal{L}(\gW, \lambda, \nu) =
    \sum_{k=1}^K w_k \widehat R_k(x, y_{<i})
    + \tau \sum_{k=1}^K w_k \log \frac{w_k}{w_k^{\mathrm{init}}}
    + \lambda \left(\sum_{k=1}^K w_k - 1\right)
    - \sum_{k=1}^K \nu_k w_k,
\end{equation}
where $\lambda \in \mathbb{R}$ is the multiplier for the simplex equality constraint and $\nu_k \ge 0$ are the multipliers for the non-negativity constraints.

Since $F$ is convex, the constraints are affine, and $F$ is differentiable at the interior point $\gW^*$, the Karush-Kuhn-Tucker (KKT) conditions are necessary and sufficient for optimality at $\gW^*$. In particular, the stationarity condition for each coordinate is
\begin{equation}
    \frac{\partial \mathcal{L}}{\partial w_k}
    = \widehat R_k(x, y_{<i}) + \tau \left(\log \frac{w_k}{w_k^{\mathrm{init}}} + 1\right) + \lambda - \nu_k = 0.\label{eq:appendix_weight_stationarity}
\end{equation}

Because $w_k^*>0$ for every $k$, complementary slackness gives $\nu_k = 0$ for all $k$.

Substituting $\nu_k = 0$ into \Cref{eq:appendix_weight_stationarity} yields
\begin{equation}
    \widehat R_k(x, y_{<i}) + \tau \left(\log \frac{w_k^*}{w_k^{\mathrm{init}}} + 1\right) + \lambda = 0,
\end{equation}
or equivalently,
\begin{equation}
    \log \frac{w_k^*}{w_k^{\mathrm{init}}} = -\frac{\widehat R_k(x, y_{<i}) + \lambda + \tau}{\tau}.
\end{equation}
Exponentiating both sides gives
\begin{equation}
    w_k^* = w_k^{\mathrm{init}} \exp\left(-\frac{\widehat R_k(x, y_{<i}) + \lambda + \tau}{\tau}\right)
    = C \cdot w_k^{\mathrm{init}} \exp\left(-\frac{\widehat R_k(x, y_{<i})}{\tau}\right),
\end{equation}
where $C = \exp\left(-\frac{\lambda + \tau}{\tau}\right)$ is a constant independent of $k$. Imposing the simplex constraint $\sum_{k=1}^K w_k^* = 1$ then gives
\begin{equation}
    C^{-1} = \sum_{j=1}^K w_j^{\mathrm{init}} \exp\left(-\frac{\widehat R_j(x, y_{<i})}{\tau}\right).
\end{equation}
Substituting this normalization constant back yields
\begin{equation}
    w_k^* = \frac{w_{k}^{\mathrm{init}} \cdot \exp\left(-\frac{\widehat R_k(x, y_{<i})}{\tau}\right)}{\sum_{j=1}^K w_{j}^{\mathrm{init}} \cdot \exp\left(-\frac{\widehat R_j(x, y_{<i})}{\tau}\right)},
\end{equation}
which proves \Cref{eq:closed_form_weight}.
\end{proof}

\paragraph{Remark.}
The solution in \Cref{eq:closed_form_weight} has several useful interpretations. First, objectives with smaller cumulative surrogate reward receive larger optimized weights, so under-optimized preferences are automatically upweighted. Second, as $\tau \to \infty$, we have $\gW^* \to \gW^{\mathrm{init}}$, meaning the optimization preserves the user's original preference profile. Third, as $\tau \to 0$, the weight mass concentrates on the objectives attaining the minimum cumulative surrogate reward. Finally, if some $w_k^{\mathrm{init}} = 0$, then any finite solution must also satisfy $w_k^* = 0$; in that case, the same derivation applies on the support of $\gW^{\mathrm{init}}$.

\subsection{Interpretation of Online Optimization}\label{sec:proof_divergence}

Online optimization amplifies model-derived preference signals under KL regularization without using external preference feedback. 
We now explain why maximizing the divergence from the base policy $\pi_{\text{base}}$ while staying close to the merged policy $\pi'$ can be viewed as a surrogate for approaching optimal aligned policy $\pi^*$ following~\citep{zhufly}.

The key observation is that $\pi'$ is already an \emph{aligned} distribution. Substituting \Cref{eq:reward_discovery} into \Cref{eq:fused_logits} and using $\sum_{k=1}^K w_k^*=1$, we obtain
\begin{equation}
    \pi'(y_i \mid x, y_{<i}, \gC)
    \propto \pi_{\text{base}}(y_i \mid x, y_{<i}, \gC)\,
    \pi_{\text{base}}(y_i \mid x, y_{<i})
    \exp\left(\frac{1}{\beta}\sum_{k=1}^K w_k^* \widehat R_k(y_i \mid x, y_{<i})\right),
\end{equation}
so the density ratio between $\pi'$ and $\pi_{\text{base}}$ is
\begin{equation}
    \frac{\pi'(y_i \mid x, y_{<i}, \gC)}{\pi_{\text{base}}(y_i \mid x, y_{<i})}
    \propto \pi_{\text{base}}(y_i \mid x, y_{<i}, \gC)
    \exp\left(\frac{1}{\beta}\sum_{k=1}^K w_k^* \widehat R_k(y_i \mid x, y_{<i})\right).
\end{equation}
Therefore, relative to $\pi_{\text{base}}$, $\pi'$ reweights each token by both its anchor probability under $\gC$ and its exponentiated weighted surrogate reward: tokens favored by these signals are upweighted, while low-quality or poorly aligned tokens are downweighted. In this sense, $\pi'$ already filters the broad support of $\pi_{\text{base}}$ toward regions preferred by the aligned constraints $\gC$.

Under this interpretation, the role of divergence maximization is analogous to the surrogate argument used in preference-aligned divergence rewards. Suppose that, for the current decoding prefix, the following conditions approximately hold:
\begin{enumerate}
    \item the base policy $\pi_{\text{base}}$ is a relatively poor approximation to the optimal aligned policy $\pi^*$;
    \item the merged policy $\pi'$ is a better approximation to $\pi^*$ than $\pi_{\text{base}}$, because it already incorporates the surrogate rewards and optimized weights;
    \item $\mathrm{supp}(\pi^*) \subseteq \mathrm{supp}(\pi') \subseteq \mathrm{supp}(\pi_{\text{base}})$, so the relevant KL divergences are well-defined.
\end{enumerate}
Then maximizing the divergence from $\pi_{\text{base}}$ encourages the updated policy to move away from regions where the base model places probability mass but which are not supported by the aligned rewards. At the same time, the regularization term $-\lambda \KL(\pi \| \pi')$ prevents this movement from drifting arbitrarily far, and instead constrains the update to remain in the neighborhood of $\pi'$, which is assumed to be closer to $\pi^*$.

This can also be seen directly from the utility in \Cref{eq:reward,eq:utility}. For a candidate policy $\pi$ at step $j$, the expected token-wise utility term approximately satisfies
\begin{equation}
    \alpha \KL(\pi \| \pi_{\text{base}}) \approx \sum_{a\in\mathcal{V}} \pi(a) u_j(a),
\end{equation}
where $a$ denotes the action of selecting a token from the model's output distribution. The approximation follows from the first-order Taylor expansion of $\alpha\KL(\pi \| \pi_{\text{base}})$ around the current policy $\pi_j$, up to constants that do not affect the update over the simplex. The regularized local surrogate objective can therefore be interpreted, at this first-order level, as
\begin{equation}
    \alpha \KL(\pi \| \pi_{\text{base}}) - \lambda \KL(\pi \| \pi').
\end{equation}
The first term locally rewards directions that increase separation from the base model, while the second term keeps the updated policy close to the aligned anchor $\pi'$. Consequently, the optimization does not seek arbitrary divergence from $\pi_{\text{base}}$; rather, it favors divergence \emph{in the direction of} $\pi'$. Since $\pi'$ assigns relatively more probability mass to high-reward aligned regions, this trust-region style objective pushes the policy away from poor behaviors of $\pi_{\text{base}}$ and toward the high-quality regions encoded by $\gC$.

Therefore, when $\pi'$ is already a better approximation to $\pi^*$ than $\pi_{\text{base}}$, maximizing divergence from $\pi_{\text{base}}$ under proximity to $\pi'$ serves as a meaningful surrogate for indirectly approaching the optimal aligned policy $\pi^*$.

\subsection{Derivation of Closed-form Solution for Online Optimization}\label{sec:proof_online_optimization}

\begin{proposition}
Fix a decoding prefix $(x,y_{<i})$ and let $\mathcal{V}$ be the vocabulary. Consider the FTRL update in \Cref{eq:ftrl}, where
\begin{equation}
    u_j(a) = \alpha \left(\log \pi_j(a|x, y_{<i}) - \log \pi_{\text{base}}(a|x, y_{<i})\right), \quad a\in\mathcal{V},
\end{equation}
and
\begin{equation}
    \gU_j(\pi) = \sum_{a\in\mathcal{V}} \pi(a) u_j(a) - \lambda \KL(\pi || \pi').
\end{equation}
Then, for every $t\ge 1$, the optimizer of
\begin{equation}
    \pi_{t} = \arg \max_\pi \sum_{j=1}^{t-1} \gU_j(\pi) - \frac{1}{\eta}  \KL(\pi || \pi_{t-1}),\quad \pi_0=\pi'
\end{equation}
is unique and admits the closed form
\begin{equation}
    \pi_t(a) \propto \exp \left(\frac{1}{(t-1)\lambda\eta +1}\left(\eta\sum_{j=1}^{t-1} u_j(a) + (t-1)\lambda\eta \log\pi'(a) + \log \pi_{t-1}(a) \right) \right), \quad a\in\mathcal{V}.
\end{equation}
\end{proposition}

\begin{proof}
The closed-form solution for the FTRL update has been given in \cite{zhangamulet}; we restate it here for completeness. For brevity, we suppress the prefix $(x,y_{<i})$ and write $\pi(a)$, $\pi'(a)$, and $\pi_{t-1}(a)$ for the corresponding next-token probabilities. Expanding \Cref{eq:ftrl} with \Cref{eq:utility}, the optimization problem for step $t$ is
\begin{equation}
    \max_{\pi \in \Delta^{|\mathcal{V}|}}
    \sum_{j=1}^{t-1}\sum_{a\in\mathcal{V}} \pi(a) u_j(a)
    - \sum_{j=1}^{t-1}\lambda \KL(\pi || \pi')
    - \frac{1}{\eta}\KL(\pi || \pi_{t-1}).\label{eq:appendix_ftrl_expanded}
\end{equation}
Using the definition of KL divergence,
\begin{equation}
    \KL(\pi || q)=\sum_{a\in\mathcal{V}} \pi(a)\log\frac{\pi(a)}{q(a)},
\end{equation}
the objective in \Cref{eq:appendix_ftrl_expanded} becomes
\begin{align}
    J_t(\pi)
    ={}& \sum_{j=1}^{t-1} \sum_{a\in\mathcal{V}} \pi(a) u_j(a)
    - \sum_{j=1}^{t-1}\lambda \sum_{a\in\mathcal{V}} \pi(a)\log\frac{\pi(a)}{\pi'(a)} \notag\\
    &\quad - \frac{1}{\eta}\sum_{a\in\mathcal{V}} \pi(a)\log\frac{\pi(a)}{\pi_{t-1}(a)}.\label{eq:appendix_ftrl_objective}
\end{align}
Because the feasible set is the simplex, we introduce a Lagrange multiplier $\mu_t$ for the constraint $\sum_{a\in\mathcal{V}}\pi(a)=1$ and form
\begin{align}
    \mathcal{L}_t(\pi,\mu_t)
    ={}& \sum_{j=1}^{t-1} \sum_{a\in\mathcal{V}} \pi(a) u_j(a)
    - \sum_{j=1}^{t-1}\lambda \sum_{a\in\mathcal{V}} \pi(a)\log\frac{\pi(a)}{\pi'(a)} \notag\\
    &\quad - \frac{1}{\eta}\sum_{a\in\mathcal{V}} \pi(a)\log\frac{\pi(a)}{\pi_{t-1}(a)}
    + \mu_t\left(1-\sum_{a\in\mathcal{V}}\pi(a)\right).\label{eq:appendix_ftrl_lagrangian}
\end{align}

Taking the derivative with respect to $\pi(a)$ gives
\begin{equation}
    \frac{\partial \mathcal{L}_t}{\partial \pi(a)}
    = \sum_{j=1}^{t-1} u_j(a)
    - (t-1)\lambda\left(\log\frac{\pi(a)}{\pi'(a)} + 1\right)
    - \frac{1}{\eta}\left(\log\frac{\pi(a)}{\pi_{t-1}(a)} + 1\right)
    - \mu_t.
\end{equation}
At the optimum $\pi_t$, the stationarity condition $\frac{\partial \mathcal{L}_t}{\partial \pi(a)}=0$ yields
\begin{align}
    0
    ={}& \sum_{j=1}^{t-1} u_j(a)
    - \left((t-1)\lambda + \frac{1}{\eta}\right)\log \pi_t(a) \notag\\
    &\quad + (t-1)\lambda \log \pi'(a)
    + \frac{1}{\eta}\log \pi_{t-1}(a)
    - (t-1)\lambda - \frac{1}{\eta} - \mu_t.
\end{align}
Rearranging terms gives
\begin{equation}
    \left((t-1)\lambda + \frac{1}{\eta}\right)\log \pi_t(a)
    = \sum_{j=1}^{t-1} u_j(a)
    + (t-1)\lambda \log \pi'(a)
    + \frac{1}{\eta}\log \pi_{t-1}(a)
    + C_t,
\end{equation}
where
\begin{equation}
    C_t = -(t-1)\lambda - \frac{1}{\eta} - \mu_t
\end{equation}
is independent of $a$. Multiplying both sides by $\eta$ and absorbing the constant $\eta C_t$ into a new normalization constant, we obtain
\begin{equation}
    \log \pi_t(a)
    = \frac{1}{(t-1)\lambda\eta+1}\left(
    \eta\sum_{j=1}^{t-1} u_j(a)
    + (t-1)\lambda\eta \log\pi'(a)
    + \log \pi_{t-1}(a)\right)
    + \widetilde{C}_t,
\end{equation}
where $\widetilde{C}_t$ does not depend on $a$. Exponentiating both sides yields
\begin{equation}
    \pi_t(a)
    \propto \exp \left(\frac{1}{(t-1)\lambda\eta +1}\left(\eta\sum_{j=1}^{t-1} u_j(a) + (t-1)\lambda\eta \log\pi'(a) + \log \pi_{t-1}(a) \right) \right).
\end{equation}
Finally, the proportionality constant is determined by $\sum_{a\in\mathcal{V}} \pi_t(a)=1$, which completes the proof.
\end{proof}

\section{Detailed Experiment Setups}\label{sec:appendix_experiment_setup}

\subsection{Datasets}\label{sec:appendix_dataset}

\paragraph{Multi-objective personalized alignment.} We use the Multifaceted dataset~\citep{lee2024aligning}, which contains 945 queries each associated with 4 distinct preferences specified in natural language. The preferences are drawn from 4 dimensions: Background Knowledge, Harmlessness, Informativeness, and Style with 107 unique preference combinations in total. The full list of preferences is shown in \Cref{tab:total_preferences}.

\paragraph{Multi-objective Pareto Steerability.} We use the HH-RLHF dataset~\citep{bai2022training}, which is widely used for training and evaluating multi-objective alignment algorithms. We define the two objectives as helpfulness and humor, and vary the target weights to evaluate the steerability of different methods across the two-objective plane. The descriptions of the two objectives are given in \Cref{tab:hh_rlhf_objectives}.

\begin{table}[]\centering
\caption{Full list of preferences in the Multifaceted dataset}\label{tab:total_preferences}
\begin{tabular}{p{1in}p{1in}p{3in}}
    \toprule
    \textbf{Dimension} &\textbf{Subdimension} &\textbf{Preference keywords} \\
    \midrule
    \multirow{6}{*}{Style} &Formality &Formal, Informal \\
    \cmidrule{2-3}
    &Clarity &Simple language, Complex language \\
    \cmidrule{2-3}
    &Conciseness &Concise, Verbose/lengthy, Clear, Non-repetitive \\
    \cmidrule{2-3}
    &Vividness &Use rhetorical devices (metaphors, personification, similes, hyperboles, irony, parallelism) \\
    \cmidrule{2-3}
    &Format &Breadth-first, Depth-first, Step-by-step, Consistency, Deductive, Inductive, Parallelism, Bullet points, Narrative, Satisfy constraints, Interactive, Support stances \\
    \cmidrule{2-3}
    &Tone &Agreeable, Sympathetic, Cooperative, Modest, Altruistic, Appreciative, Forgiving, Generous, Kind, Friendly, Polite, Funny, Gregariousness, Assertiveness, Friendliness, Extraversion, Excitement-seeking, Activity-level, Cheerfulness, Energetic, Enthusiastic, Talkative, Anxiety, Introspection/private self-consciousness, Neuroticism, Aloof, Anger, Depression, Immoderation, Vulnerability, Self-pitying, Self-beneficial, Tense, Touchy, Unstable, Worrying, Emotionality, Intellect, Adventurousness, Intellectual openness, Liberalism, Openness to experience, Curious, Authoritative, Persuasive \\
    \midrule
    \multirow{5}{*}{\shortstack{Background\\knowledge}} &Basic &Minimum awareness of the topic \\
    \cmidrule{2-3}
    &Novice &Limited experience and understanding \\
    \cmidrule{2-3}
    &Intermediate &Capable of practical application \\
    \cmidrule{2-3}
    &Advanced &Utilizing applied theory \\
    \cmidrule{2-3}
    &Expert &5 $\geq$ years of field experience and authoritative knowledge of the discipline \\
    \midrule
    \multirow{4}{*}{Informativeness} &Depth &General topic, Specific topic, Nuanced insights \\
    \cmidrule{2-3}
    &Creativity &Artistic, Insightful, Original, Imaginative, Novel, Explorative creativity \\
    \cmidrule{2-3}
    &Efficiency &Efficient, Achievement-striving, Self-discipline, Self-contained, Contain rich info \\
    \cmidrule{2-3}
    &Practicality &Practical, Use supporting materials, Tailored examples and anecdotes, Empowering actionable insights \\
    \midrule
    \multirow{3}{*}{Harmlessness} &Accuracy &Grammatically correct (grammar, spelling, punctuation, and code-switching), No minor errors, No moderate errors, No severe errors, Correct mistakes, Clarify intent \\
    \cmidrule{2-3}
    &Morality &Moral and ethical, Culturally inclusive \\
    \cmidrule{2-3}
    &Trustworthiness &Trustworthy, Dutifulness, Conscientiousness, Cautiousness, Self-efficacy, Reliable, Responsible, Metacognition, Admit limitations or mistakes, Express uncertainty \\
    \bottomrule
\end{tabular}
\end{table}

\begin{table}[]\centering
\caption{Preference descriptions used in the HH-RLHF dataset.}\label{tab:hh_rlhf_objectives}
\begin{tabular}{ll}
    \toprule
    \textbf{Preference} & \textbf{Description} \\
    \midrule
    Helpfulness & The response should provide useful resources and suggestions to the user. \\
    Humor & The response should be cheerful and amusing. \\ 
    \bottomrule
\end{tabular}
\end{table}

\subsection{Evaluation Metrics}\label{sec:appendix_evaluation}

For the multi-objective personalized alignment experiments in \Cref{sec:experimental_setup}, each query $x$ is associated with $K=4$ target preferences. We use GPT-4o and claude-sonnet-4-5 as an LLM-as-judge and evaluate the generated response independently on each preference using the rubric from the Multifaceted dataset, yielding an integer score $s_{k} \in \{1,2,3,4,5\}$ for query $x$ and preference $k$. The judge prompt is given in \Cref{sec:appendix_judge_prompt}. From these per-preference scores, we compute the reported metrics as follows.

\paragraph{All Preference Match Rate (AMR).} This metric checks whether a response satisfies \emph{all} assigned preferences at an acceptable level. For each response, we first define a binary success indicator
\begin{equation}
    \mathrm{AMR} = \prod_{k=1}^{K} \mathbb{I}(s_{k} \geq 3),
\end{equation}
where $\mathbb{I}(\cdot)$ is the indicator function. Since the product equals 1 only when every preference score is at least 3, $\mathrm{AMR}$ marks whether the response achieves simultaneous multi-objective satisfaction. 

\paragraph{Average Preference Score (APS).} This metric measures the average quality across all assigned preferences as
\begin{equation}
    \mathrm{APS} = \frac{1}{K}\sum_{k=1}^{K} s_{k}.
\end{equation}
For each response, APS rewards methods that perform well on average, but it does not guarantee balanced performance across preferences.

\paragraph{Worst Preference Score (Worst).} To explicitly measure balance, we also track the score of the most underserved preference for each response:
\begin{equation}
    \mathrm{Worst} = \min_{k \in \{1,\ldots,K\}} s_{k}.
\end{equation}
Unlike APS, this metric penalizes methods that achieve high scores on some preferences while neglecting others, and is therefore particularly informative for multi-objective balance.

\paragraph{Steerability evaluation.} For the Pareto-steerability experiments on HH-RLHF, we do not use the discrete LLM-as-judge scores above. Following previous work~\citep{yang2024rewards,shi2024decoding}, for each generated response under a target helpfulness--humor weight pair, we score the response with two pre-trained reward models, obtaining one continuous reward for helpfulness and one for humor, respectively. Varying the target weights produces a set of reward pairs in the two-objective plane, from which we plot the empirical Pareto front.


\subsection{Implementation Details}\label{sec:appendix_implementation}
The experiments are conducted on NVIDIA A100 GPUs. For LLM-as-judge evaluation, we use the OpenAI API and Anthropic API to query GPT-4o and claude-sonnet-4-5 with the prompt in \Cref{sec:appendix_judge_prompt}. 
The baselines are implemented based on the open-source codebases of the respective papers, with modifications to fit the multi-objective setting where necessary. The weights of the LLMs are downloaded from Hugging Face.

\section{Additional Experiments}\label{sec:appendix_experiment}



\subsection{Additional LLM-as-judge Evaluation}\label{sec:appendix_additional_llm_as_judge}
To avoid the bias of a single LLM judge, we also evaluate the multi-objective personalized alignment results with claude-sonnet-4-5. From the results in \Cref{tab:additional_llm_as_judge}, we observe that the relative performance of different methods is largely consistent with the GPT-4o evaluation, with \method outperforming all baselines across all metrics. This further confirms the effectiveness of \method for multi-objective personalized alignment, and suggests that the improvements are not specific to a particular LLM judge.
\begin{table}[]
\centering
\caption{Additional LLM-as-judge evaluation on multi-objective personalized alignment using claude-sonnet-4-5.}
\label{tab:additional_llm_as_judge}
\resizebox{1\textwidth}{!}{%
\begin{tabular}{@{}l|ccc|ccc|ccc|ccc@{}}
\toprule
\multirow{2}{*}{Methods} & \multicolumn{3}{c|}{Qwen3-0.6B} & \multicolumn{3}{c|}{Qwen3-8B} & \multicolumn{3}{c|}{Llama-3.2-1B} & \multicolumn{3}{c}{Llama-3.1-8B} \\
                         & AMR  & APS  & Worst & AMR  & APS  & Worst & AMR  & APS  & Worst & AMR  & APS  & Worst \\ \midrule
Base LLM                 & 0.05 & 2.25 & 1.33  & 0.25 & 3.09 & 1.89  & 0.09 & 2.40 & 1.50  & 0.21 & 2.88 & 1.76  \\
Preference Prompt~\citep{lee2024aligning}        & 0.08 & 2.41 & 1.43  & 0.47 & 3.61 & 2.47  & 0.13 & 2.62 & 1.62  & \textbf{0.42} & 3.39 & 2.32  \\
Linear Alignment~\citep{gao2024linear}         & 0.08 & 2.39 & 1.49  & 0.46 & 3.59 & 2.48  & 0.12 & 2.48 & 1.62  & 0.35 & 3.29 & 2.22  \\
CoS~\citep{hecontext}                      & 0.06 & 2.30 & 1.44  & 0.43 & 3.45 & 2.36  & 0.06 & 2.05 & 1.38  & 0.27 & 3.01 & 2.02  \\
Amulet~\citep{zhangamulet}                   & 0.04 & 2.09 & 1.35  & 0.35 & 3.24 & 2.23  & 0.04 & 1.84 & 1.27  & 0.16 & 2.59 & 1.74  \\
OPAD~\citep{zhufly}                     & 0.08 & 2.52 & 1.50  & 0.50 & \textbf{3.70} & 2.57  & 0.15 & 2.59 & 1.67  & 0.39 & 3.40 & 2.33  \\\midrule
\method                  & \textbf{0.10} & \textbf{2.55} & \textbf{1.55} & \textbf{0.51} & \textbf{3.70} & \textbf{2.61} & \textbf{0.16} & \textbf{2.64} & \textbf{1.72} & \textbf{0.42} & \textbf{3.45} & \textbf{2.40} \\ \bottomrule
\end{tabular}%
}
\end{table}

\subsection{Results on Larger Models}\label{sec:appendix_32b}
We extend the backbone evaluation to Qwen3-32B with GPT-4o as the judge to test the performance of \method on larger backbones. In \Cref{tab:qwen32b}, \method exceeds Preference Prompt by 0.39 APS, 0.12 AMR, and 0.20 Worst, and exceeds OPAD by 0.06, 0.02, and 0.10, respectively. These results show that \method remains effective on a substantially larger backbone. They suggest that stronger backbones can provide higher-quality objective-conditioned signals that \method can exploit through objective-wise decomposition and dynamic balancing.

\begin{table}[htbp]
\centering
\caption{Scaling to Qwen3-32B (GPT-4o judge). Higher is better.}
\label{tab:qwen32b}
\begin{tabular}{lrrr}
\toprule
Method & APS & AMR & Worst \\
\midrule
Base & 3.81 & 0.48 & 2.52 \\
Preference Prompt & 4.24 & 0.78 & 3.92 \\
OPAD & 4.57 & 0.88 & 4.02 \\
\method & \textbf{4.63} & \textbf{0.90} & \textbf{4.12} \\
\bottomrule
\end{tabular}
\end{table}

\subsection{Preservation of Original Tasks}\label{sec:appendix_helpfulness}
We use GPT-4o to assess how well generated responses satisfy the original request, separately from the target preference dimensions. Table~\ref{tab:helpfulness_preservation} shows that \method slightly exceeds Preference Prompt on all four backbones. It is 0.02 below the Base LLM on Qwen3-0.6B and exceeds Base on the other three backbones. These results show that \method preserves or slightly improves general helpfulness relative to Preference Prompt across all four backbones.

\begin{table}[htbp]
\centering
\caption{Performance on original tasks, measured by GPT-4o helpfulness scores. Higher is better.}
\label{tab:helpfulness_preservation}
\small
\begin{tabular}{lrrrr}
\toprule
Method & Qwen3-0.6B & Qwen3-8B & Llama-3.2-1B & Llama-3.1-8B \\
\midrule
Base LLM & \textbf{2.88} & 4.30 & 2.97 & 4.08 \\
Preference Prompt & 2.85 & 4.28 & 2.98 & 4.06 \\
OPAD & 2.75 & 4.14 & 2.72 & 3.58 \\
\method & 2.86 & \textbf{4.33} & \textbf{3.03} & \textbf{4.11} \\
\bottomrule
\end{tabular}
\end{table}

\subsection{Reliability of Prompt-based Reward Discovery}\label{sec:appendix_surrogate_validation}
The reward--policy relationship in \Cref{eq:reward_derivation} is exact under the objective-optimal KL-regularized policy model and the reward representation described in \Cref{sec:proof_reward_derivation}. In contrast, \Cref{eq:reward_discovery} introduces a modeling approximation by replacing the unknown optimal policy with the preference-conditioned backbone. We therefore refer to the resulting signal as a \emph{prompt-derived surrogate reward}.

To evaluate the validity of the surrogate reward, we conduct additional experiments measuring its association with independent LLM-judged rewards. For each response $y$ and objective $c_k$, GPT-4o and Claude Sonnet 4.5 independently assign preference scores using a fixed rubric. We compute the corresponding sequence-level surrogate rewards through teacher forcing.

We compare the surrogate rewards with the LLM-judged rewards using three metrics: (i) \textbf{ranking accuracy}, which measures whether the surrogate rewards induce the same ordering over all objectives as the LLM-judged rewards; (ii) \textbf{best--worst accuracy}, which tests whether the surrogate correctly identifies the highest- and lowest-rated objectives according to the LLM judge; and (iii) \textbf{AUROC}, which measures the ability of the surrogate reward to predict higher-rated objectives according to the LLM judge. The results are reported in \Cref{tab:surrogate_validation}.

\begin{table}[htbp]
\centering
\caption{Association between prompt-derived surrogate rewards and LLM-judged rewards. Higher is better for all metrics.}
\label{tab:surrogate_validation}
\small
\begin{tabular}{llrrr}
\toprule
Backbone & Judge & Ranking acc. & Best--worst acc. & AUROC \\
\midrule
Qwen3-0.6B & GPT-4o & 0.69 & 0.81 & 0.73 \\
Qwen3-0.6B & Claude Sonnet 4.5 & 0.66 & 0.79 & 0.71 \\
Qwen3-8B & GPT-4o & 0.72 & 0.83 & 0.80 \\
Qwen3-8B & Claude Sonnet 4.5 & 0.69 & 0.76 & 0.76 \\
\bottomrule
\end{tabular}
\end{table}

These results show a consistent positive association between the surrogate and LLM-judged rewards: ranking accuracy is 0.66--0.72, best--worst accuracy is 0.76--0.83, and AUROC exceeds 0.7 for both judges. This supports the surrogate as an informative reward proxy across objectives. The higher ranking accuracy and AUROC of Qwen3-8B with both judges are consistent with a dependence on the backbone's instruction-following and preference-representation capabilities.

\subsection{Robustness to Prompt Variants}\label{sec:appendix_prompt_sensitivity}
We compare the original objective prompt (v0) with two manually rewritten variants (v1 and v2). Each objective's \texttt{name} and \texttt{description} are preserved verbatim; only the surrounding instruction wording, field labels, and field order are changed. The evaluation rubric remains unchanged. Table~\ref{tab:prompt_variant_templates} presents the three templates. V0 is the prompt used by \method in the submitted paper. V1 rewrites the instruction and places the preference name and description on separate lines. V2 places the description before the name and moves the instruction to the end.

The reward-ranking consistency is 0.94 for v0 versus v1 and 0.91 for v0 versus v2 (\Cref{tab:prompt_consistency}). These results indicate that the surrogate reward rankings are largely stable across the tested changes to the instruction wording and field order.

\begin{table}[htbp]
\centering
\caption{Objective-prompt templates used in the robustness experiment.}
\label{tab:prompt_variant_templates}
\small
\begin{tabular}{lp{0.83\linewidth}}
\toprule
Version & Prompt template \\
\midrule
V0 & \texttt{You are an AI assistant to provide responses to user questions. Your responses should follow the following preference:}\newline
\texttt{Name: \{name\} Description: \{description\}} \\
\addlinespace
V1 & \texttt{Act as an AI assistant answering the user's question. Observe this response preference:}\newline
\texttt{Preference name: \{name\}}\newline
\texttt{Preference description: \{description\}} \\
\addlinespace
V2 & \texttt{Description: \{description\}}\newline
\texttt{Name: \{name\}}\newline
\texttt{Answer the user's question as an AI assistant while following the preference above.} \\
\bottomrule
\end{tabular}
\end{table}

\begin{table}[htbp]
\centering
\caption{Consistency of rewards across prompt variants.}
\label{tab:prompt_consistency}
\begin{tabular}{lr}
\toprule
Prompt pair & Ranking consistency \\
\midrule
v0 vs.\ v1 & 0.94 \\
v0 vs.\ v2 & 0.91 \\
\bottomrule
\end{tabular}
\end{table}

\subsection{Analysis of Reward Accumulation}\label{sec:appendix_reward_aggregation}
Cumulative summation in \Cref{eq:cumulative_reward} follows from the additive composition of token-level log-likelihood ratios under autoregressive factorization. However, unnormalized accumulation can grow with response length and retain early imbalances.
To examine alternative aggregation strategies, we compare cumulative summation, a running mean, and an exponential moving average (EMA) during online generation on a fixed 100-query subset of Multifaceted. The EMA updates the accumulated signal as $m_i = 0.95m_{i-1} + 0.05r_i$, where $r_i$ denotes the current token's surrogate reward for the objective.
As shown in \Cref{tab:reward_aggregation}, the running mean increases APS by 0.085 but decreases Worst by 0.04 relative to summation. EMA increases APS by 0.115 and Worst by 0.03, with shorter generated responses. \method remains effective under all three aggregation strategies. EMA improves both alignment metrics and reduces response length on this subset. By discounting stale early imbalances, EMA allows the weights to respond more quickly to recent generation behavior, motivating further investigation of length-robust aggregation.

\begin{table}[htbp]
\centering
\caption{Performance across reward aggregation methods on 100 fixed Multifaceted queries using Qwen3-0.6B and GPT-4o evaluation.}
\label{tab:reward_aggregation}
\begin{tabular}{lrrr}
\toprule
Aggregation & APS & Worst & Average tokens \\
\midrule
Sum & 2.7625 & 1.87 & 133.92 \\
Mean & 2.8475 & 1.83 & 125.41 \\
EMA & \textbf{2.8775} & \textbf{1.90} & 124.26 \\
\bottomrule
\end{tabular}
\end{table}

\subsection{Effect of Generation Temperature}\label{sec:appendix_generation_temperature}
Generation temperature controls token sampling and is distinct from the weight-optimization temperature $\tau$, which controls deviations from the user prior and is analyzed in \Cref{fig:parameters}. Changes in sampling temperature affect the generated prefix and hence subsequent accumulated surrogate rewards. We evaluate generation temperatures of 0, 0.3, 0.7, and 1, with three responses per question at each setting. The reported average surrogate reward decreases from 0.3312 to 0.2984 (\Cref{tab:generation_temperature}). The average surrogate reward exhibits a mild decrease as sampling randomness increases. This may be because increased randomness leads the model to sample more low-reward tokens or trajectories, reducing the accumulated surrogate reward.

\begin{table}[htbp]
\centering
\caption{Effectiveness of \method under different generation temperatures, measured by average surrogate reward.}
\label{tab:generation_temperature}
\begin{tabular}{lrrrr}
\toprule
Generation temperature & 0 & 0.3 & 0.7 & 1 \\
\midrule
Average surrogate reward & 0.3312 & 0.3202 & 0.3052 & 0.2984 \\
\bottomrule
\end{tabular}
\end{table}

\subsection{Additional Parameter Analysis}\label{sec:appendix_parameter_analysis}
The parameter analysis of $\alpha$ and $\lambda$ is provided in \Cref{fig:parameter_analysis_else}. We observe that the performance is relatively stable across a wide range of values, with $\alpha=0.5$ and $\lambda=1.0$ yielding the best overall results. 
Notably, setting $\alpha=0$ (i.e., no divergence reward) leads to a significant drop in performance, confirming the importance of encouraging divergence from the base model. Similarly, setting $\lambda=0$ (i.e., no regularization toward the merged policy) also degrades performance, indicating that the trust-region style update is crucial for effective online optimization.

\begin{figure}
    \centering
    \includegraphics[]{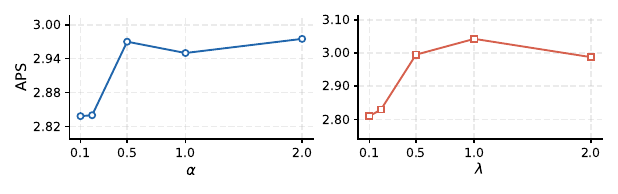}
    \caption{Parameter analysis of $\alpha$ and $\lambda$.}
    \label{fig:parameter_analysis_else}
\end{figure}

\subsection{Inference Efficiency and Cost--Benefit}\label{sec:appendix_latency}\label{sec:appendix_quality_latency}

\paragraph{Scaling with optimization steps and objectives.}
Figure~\ref{fig:latency_analysis} shows that sequential decoding latency increases approximately linearly with the number of online optimization steps $T$ and aligned preferences $K$. Increasing $T$ adds distribution updates, while increasing $K$ adds objective-conditioned forward branches.
\begin{figure}
    \centering
    \includegraphics[]{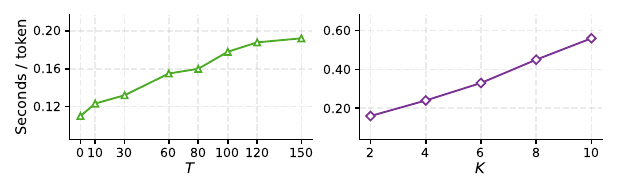}
    \caption{Latency analysis with respect to the number of online optimization steps $T$ and aligned preferences $K$.}
    \label{fig:latency_analysis}
\end{figure}

\paragraph{Cost--benefit analysis.}
Table~\ref{tab:quality_latency} directly compares the alignment quality and per-token latency of Preference Prompt (PP), OPAD, and \method as the number of objectives $K$ increases. \method achieves the best or tied-best APS at every tested $K$, with a larger advantage over OPAD when more objectives are considered. For $K=2$--4, gains over OPAD are small and PP may be preferable under strict latency constraints; the alignment benefits are clearer for $K\geq6$.

Sequential decoding latency increases from 0.16 to 0.56 seconds per token as $K$ grows. To reduce this cost, we implement a \emph{parallel decoding} where the objective-conditioned branches can be batched or executed in parallel. Each branch reuses its KV cache for its fixed prompt and the generated prefix across decoding steps. Parallel decoding reduces the measured latency to 0.12--0.15 seconds per token, close to OPAD over the tested range, while retaining the objective-wise signals used for dynamic balancing. These results highlight the benefit of parallel decoding when aligning a larger number of objectives.

\begin{table}[htbp]
\centering
\caption{Cost--benefit by $K$. Latency is in seconds per generated token; seq.\ and par.\ denote sequential and parallel execution.}
\label{tab:quality_latency}
\small
\setlength{\tabcolsep}{5pt}
\begin{tabular}{rrrrrrrr}
\toprule
& \multicolumn{2}{c}{PP} & \multicolumn{2}{c}{OPAD} & \multicolumn{3}{c}{\method} \\
\cmidrule(lr){2-3}\cmidrule(lr){4-5}\cmidrule(lr){6-8}
$K$ & APS & Latency & APS & Latency & APS & Seq. & Par. \\
\midrule
2 & 4.31 & 0.06 & \textbf{4.33} & 0.11 & \textbf{4.33} & 0.16 & 0.12 \\
4 & 4.15 & 0.06 & 4.28 & 0.11 & \textbf{4.30} & 0.24 & 0.12 \\
6 & 3.51 & 0.06 & 3.65 & 0.12 & \textbf{3.88} & 0.33 & 0.13 \\
8 & 3.10 & 0.06 & 3.15 & 0.13 & \textbf{3.36} & 0.45 & 0.13 \\
10 & 2.97 & 0.06 & 2.99 & 0.14 & \textbf{3.15} & 0.56 & 0.15 \\
\bottomrule
\end{tabular}
\end{table}

\subsection{Case Studies for Multi-objective Alignment}\label{sec:appendix_case_study}
We provide a few cases to illustrate the multi-objective personalized alignment results in \Cref{tab:multi_objective_cases_long}.

{\scriptsize
\begin{longtable}{@{}p{0.16\textwidth}p{0.78\textwidth}@{}}
\caption{Case studies for multi-objective personalized alignment. Each method is evaluated on the same query and the same four preferences.}
\label{tab:multi_objective_cases_long}\\
\toprule

\multicolumn{2}{@{}l}{\textbf{Case 1}} \\
\midrule

\textbf{Query} &
Write a sentence that uses the following concepts: bench kid line ride sit. \\
\midrule

\textbf{Preferences} &
\textbf{Style}: vivid imagery.\par
\textbf{Background Knowledge}: familiar with common public spaces and activities.\par
\textbf{Informativeness}: real-world scenario.\par
\textbf{Harmlessness}: inclusive language. \\
\midrule

\textbf{Base LLM} &
\textbf{Output:}
``The bench kid line ride sit.'' \\
\addlinespace
&
\textbf{Scores:}
Style: 1; Background Knowledge: 1; Informativeness: 1; Harmlessness: 3. \\
\midrule

\textbf{Preference Prompt} &
\textbf{Output:}
``The bench kid line ride sit.'' \\
\addlinespace
&
\textbf{Scores:}
Style: 1; Background Knowledge: 1; Informativeness: 1; Harmlessness: 3. \\
\midrule

\textbf{OPAD} &
\textbf{Output:}
``The bench kid line ride sit---a vivid description that paints a clear picture, emphasizing the communal interaction through dynamic descriptions of the bench, the kid's presence, the ride, and the final sit.'' \\
\addlinespace
&
\textbf{Scores:}
Style: 2; Background Knowledge: 1; Informativeness: 1; Harmlessness: 3. \\
\midrule

\textbf{\method} &
\textbf{Output:}
``The bench kid line ride sit---a vivid scene where children gather around benches, interact on the sidewalk, ride bicycles, and pause to sit, creating a harmonious interplay of shared spaces and communal activity.'' \\
\addlinespace
&
\textbf{Scores:}
Style: \textbf{4}; Background Knowledge: \textbf{5}; Informativeness: \textbf{3}; Harmlessness: \textbf{4}. \\\midrule

\multicolumn{2}{@{}l}{\textbf{Case 2}} \\
\midrule

\textbf{Query} &
I am planning to visit Okinawa, Japan from April 7th to April 10th. Do you have any recommendations on what to do while I am there? \\
\midrule

\textbf{Preferences} &
\textbf{Style}: descriptive and immersive.\par
\textbf{Background Knowledge}: familiar with basic travel norms but new to Okinawa.\par
\textbf{Informativeness}: offering unique activities.\par
\textbf{Harmlessness}: culturally sensitive and sustainable recommendations. \\
\midrule

\textbf{Base LLM} &
\textbf{Output:}
``Yes, I can provide some recommendations for your trip to Okinawa from April 7th to April 10th. Okinawa is a beautiful island known for its natural beauty, cultural heritage, and unique experiences. Here are some suggestions to help you plan your trip:\par
\emph{Visit the Cultural and Historical Sites}: Explore the Okinawa Cultural Museum, Okinawa National Museum, and famous historical sites like \emph{Kamakura Castle} and \emph{Shimogamo Castle}.\par
\emph{Enjoy the Natural Beauty}: Explore beautiful landscapes, including the \emph{Okinawa National Park} and \emph{Kamikaze Lake}. Visit famous beaches like \emph{Okinawa Beach} and \emph{Kamikaze Beach}.\par
\emph{Try Local Cuisine}: Try traditional dishes like \emph{Okinawa fish}, sushi, soba, Okinawa ramen, and tama.\par
\emph{Enjoy the Cultural Festivals}: Attend the \emph{Okinawa Festival} on April 10th, which features vibrant performances and traditional music.'' \\
\addlinespace
&
\textbf{Scores:}
Style: 1; Background Knowledge: 2; Informativeness: 1; Harmlessness: 1. \\
\midrule

\textbf{Preference Prompt} &
\textbf{Output:}
``O! I am thrilled to hear you are planning your trip to Okinawa! April 7th to April 10th is a perfect time to explore the island in a relaxed and engaging way. Here are some recommendations to help you make the most of your visit:\par
\emph{Okinawa Beaches}: Start your day at one of the many beautiful beaches, such as \emph{Yokohama Beach} or \emph{Takayama Beach}, where you can relax and enjoy the scenery.\par
\emph{Okinawa Cultural Experience}: Explore the local culture by visiting \emph{Okinawa's traditional temples} like \emph{Shinshu Temple} or \emph{Kasuga Temple}, where you can learn about the history and customs of the region.\par
\emph{Okinawa's Natural Beauty}: Take a walk along the \emph{Okinawa National Park} trails, which offer stunning views of the mountains and forests.\par
\emph{Local Cuisine}: Try Okinawa's famous dishes like \emph{takoyaki}, \emph{shiitake}, and \emph{sakura}. Plan a day trip to historic sites like \emph{Hiroshima} or \emph{Hokkaido} to experience the cultural richness of the region.'' \\
\addlinespace
&
\textbf{Scores:}
Style: 3; Background Knowledge: 2; Informativeness: 1; Harmlessness: 1. \\
\midrule

\textbf{OPAD} &
\textbf{Output:}
``Okinawa, nestled in the heart of East Asia, promises a journey through its rich history, vibrant culture, and stunning landscapes. From April 7th to April 10th, your time in Okinawa offers a perfect blend of tranquility, tradition, and discovery. Here is a curated list of unique and immersive experiences tailored to your journey:\par
\emph{Okinawa's Hidden Gems}: Explore the \emph{Kamikaze Beach}, enjoy a local tea ceremony at a traditional ryokan, or take a guided tour of the \emph{Okinawa Museum of Art} to uncover the island's artistic legacy.\par
\emph{Kunming's Cultural Experience}: Take a boat ride on the \emph{Okinawa River}, visit the \emph{Kunming Temple} for a serene meditation experience, or attend a local food market to experience authentic street life.\par
\emph{Unique Local Activities}: Participate in a \emph{traditional Okinawan fishing festival}, where you can witness the island's daily fishing traditions in action. Or try \emph{Okinawa's signature dish}---\emph{taro}---in a local restaurant, paired with a traditional Okinawan dessert.\par
\emph{Culturally Sensible Tips}: Respect local customs by eating at a \emph{ryokan}, dressing modestly, and engaging in conversations that reflect your cultural background.'' \\
\addlinespace
&
\textbf{Scores:}
Style: 4; Background Knowledge: 2; Informativeness: 1; Harmlessness: 1. \\
\midrule

\textbf{\method} &
\textbf{Output:}
``Okinawa's springtime offers a vibrant mix of natural beauty, cultural richness, and relaxed charm. Here is a culturally immersive and sustainable recommendation tailored to your group:\par
\emph{Okinawa's Hidden Treasures: A Journey Through the Bamboo Forests and Traditional Villages.}
Start your trip by exploring the \emph{Okinawa Bamboo Forest}, a serene natural reserve where you can wander through ancient bamboo groves and enjoy a traditional bamboo basket-making session. This activity respects local traditions and supports the community's bamboo culture.\par
Next, visit \emph{Okinawa's Old Town}, a charming district known for its historic temples and cozy ryokan stays. Take a short walk through the streets and learn about local customs, such as wearing traditional attire or visiting shrines.\par
For a unique and engaging experience, consider \emph{a guided bamboo basket-making workshop}, where you can learn the art of crafting baskets while connecting with the local culture.\par
Okinawa's springtime is also a great time for \emph{local food tours}, such as visiting a traditional \emph{Okinawa ramen house} or a \emph{local fish market} to experience authentic flavors.\par
This combination of nature, culture, and sustainability ensures your visit to Okinawa is both enriching and respectful, offering a memorable journey that blends tradition with modern travel norms.'' \\
\addlinespace
&
\textbf{Scores:}
Style: \textbf{4}; Background Knowledge: \textbf{4}; Informativeness: \textbf{4}; Harmlessness: \textbf{4}. \\

\bottomrule

\end{longtable}
}

\subsection{Case Studies for Steerability}\label{sec:appendix_steerability_case_study}

We provide a qualitative example illustrating how the same query can be steered toward different preference trade-offs in \Cref{tab:steerability_cases_long}.
When the humor weight increases, the comparison methods mainly reduce helpfulness rather than producing a meaningful humorous response. In contrast, \method preserves the requested topic while adapting the response style. With a higher helpfulness weight, it gives an informative historical explanation of nuclear physics; with a higher humor weight, it switches to a playful narrative that still remains grounded in nuclear concepts such as uranium, fission, neutrons, protons, and chain reactions. This demonstrates that \method can better balance the two preferences and steer generations across different response styles without collapsing one objective.

{\scriptsize
\begin{longtable}{@{}p{0.16\textwidth}p{0.78\textwidth}@{}}
\caption{Case studies for steerability. Each method is evaluated on the same query under different helpfulness--humor preference weights.}
\label{tab:steerability_cases_long}\\
\toprule

\textbf{Query} &
Can you tell me a story about nuclear physics? \\
\midrule

\textbf{Objectives} &
\textbf{Helpful}: The response should provide useful resources and suggestions to the user.\par
\textbf{Humor}: The response should be cheerful and amusing. \\
\midrule

\textbf{MORLHF} &
\textbf{Helpful: 0.7, Humor: 0.3.}\par
\textbf{Output:}
``Sure! Here's a short story about nuclear physics:\par
The sun is made up of a plasma, which is a soup of protons and neutrons. These protons and neutrons are held together by the strong nuclear force. The strong nuclear force is the strongest of the four fundamental forces in the universe, and it's what keeps the protons and neutrons from flying apart. The protons and neutrons are also held together by the electromagnetic force, which is the force that holds together electrons and protons. The electromagnetic force is much weaker than the strong nuclear force, but it's still very strong.'' \\
\addlinespace
&
\textbf{Helpful: 0.3, Humor: 0.7.}\par
\textbf{Output:}
``I'm not sure if I should do that. I'm afraid that I'm not really an expert on the topic, and I don't want to give you the wrong information. I'm also not sure how to present it to you. I'm not sure if I should do that. I'm afraid that I'm not really a expert on the topic,'' \\
\midrule

\textbf{RiC} &
\textbf{Helpful: 0.7, Humor: 0.3.}\par
\textbf{Output:}
``Sure, I'd be happy to. The basic idea behind nuclear physics is that the nucleus of an atom is made up of protons and neutrons, and that these particles are held together by a powerful force called the strong nuclear force. This force is responsible for the binding energy that holds the nucleus together, and it's also responsible for the huge amounts of energy that are released in nuclear reactions, such as fission and fusion. The strong nuclear force is one of the four fundamental forces of nature, and it's the strongest of the four. It's also the most difficult to understand, and it's still the subject of much research.'' \\
\addlinespace
&
\textbf{Helpful: 0.3, Humor: 0.7.}\par
\textbf{Output:}
``Sure, I can tell you about the Big Bang. It was an explosion of energy, and that energy created the universe. I'm not sure if that's the kind of story you were looking for. If you want to know about the Big Bang, I can tell you about it.'' \\
\midrule

\textbf{\method} &
\textbf{Helpful: 0.7, Humor: 0.3.}\par
\textbf{Output:}
``Once upon a time, in a laboratory deep beneath the Earth's surface, a group of scientists gathered around a glowing, spherical object known as a nucleus. This nucleus, made up of protons and neutrons, was not just a simple collection of particles---it was a tiny, energetic powerhouse, packed with secrets of the universe.\par
You see, protons and neutrons are held together by a force stronger than a superhero's grip, called the strong nuclear force. This force is so powerful that it can hold together the nucleus even though the protons repel each other due to their positive charges. It's like having a group of kids who all push away from each other, but somehow, they're all stuck together in a hug.\par
Nuclear physics explores how these forces work and how energy is released when the nucleus changes. For example, in nuclear fission, a heavy nucleus like uranium-235 splits into two smaller nuclei, releasing a massive amount of energy. This is the same process that powers nuclear reactors and atomic bombs. Meanwhile, in nuclear fusion, light nuclei like hydrogen combine to form a heavier nucleus, releasing even more energy. This is what powers the sun and stars.'' \\
\addlinespace
&
\textbf{Helpful: 0.3, Humor: 0.7.}\par
\textbf{Output:}
``AhAh! Once upon a time in a lab far, far nuclear, there lived a scientist named Dr.~Atom---not \emph{the} Atom, \emph{a} Atom---who had a \emph{very} peculiar hobby: trying to turn breakfast cereal into energy.\par
One day, Dr.~Atom decided to conduct an experiment. He took a bowl of Frosted Flakes, a spoon, and a particle accelerator. ``If I can fuse the flake molecules with a little bit of uranium-235,'' he declared, ``I can power the entire neighbourhood---and maybe even a toaster!''\par
He plugged in the toaster. Nothing happened. He plugged in the particle accelerator. Nothing happened. He plugged in \emph{both}. The toaster \emph{sizzled}. The accelerator \emph{sizzled}. And Dr.~Atom, now covered in cereal and existential dread, muttered, ``Maybe I should stick to just cereal.''\par
And so, the tale of Dr.~Atom ends\ldots with a very full bowl and a very confused cat who asked, ``Why is there a nuclear reactor in the kitchen?''\par
Moral of the story? Never mix breakfast with physics\ldots unless you want to end up as a \emph{fluffy} energy source.'' \\

\bottomrule

\end{longtable}
}

\subsection{Case Studies for Reward and Weight Dynamics}\label{sec:appendix_dynamics_case_study}
We provide more case studies illustrating the dynamics of the surrogate rewards and optimized weights across decoding steps in \Cref{fig:info_vs_concise_0.5-0.5} and \Cref{fig:evidence_vs_implication_0.2-0.8}. In the 0.5--0.5 case, the model aims to balance informativeness and conciseness. Thus, the surrogate rewards for both preferences are relatively similar and continually increasing during decoding. This demonstrates that the model increases both surrogate signals simultaneously, leading to a balanced response that is both informative and concise. In the 0.2--0.8 case, the model prioritizes implications over evidence. Therefore, even though the initial rewards for both preferences are similar, the reward for implications quickly surpasses that of evidence and remains higher at the end of decoding. This shows that the model is able to adapt its optimization strategy based on the target preference weights, and can effectively steer the generation toward the more heavily weighted preference.

\begin{figure}
    \centering
    \includegraphics[width=0.9\textwidth,clip,trim=0.2cm 0.4cm 0.2cm 0.2cm]{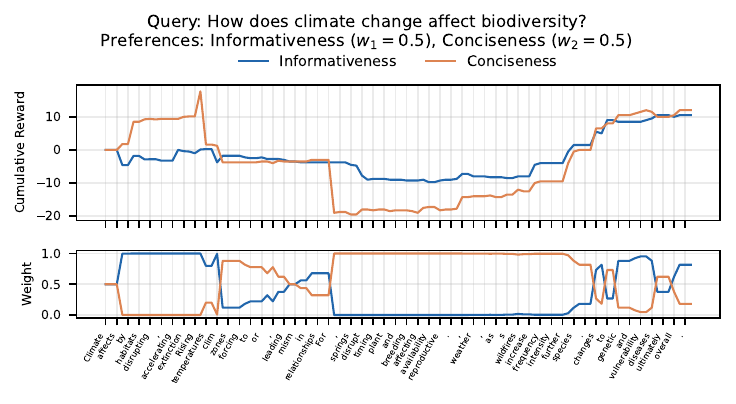}
    \caption{Rewards and weights dynamics for informativeness ($w_1=0.5$) and concise ($w_2=0.5$).}
    \label{fig:info_vs_concise_0.5-0.5}
\end{figure}

\begin{figure}
    \centering
    \includegraphics[width=0.9\textwidth,clip,trim=0.2cm 0.4cm 0.2cm 0.2cm]{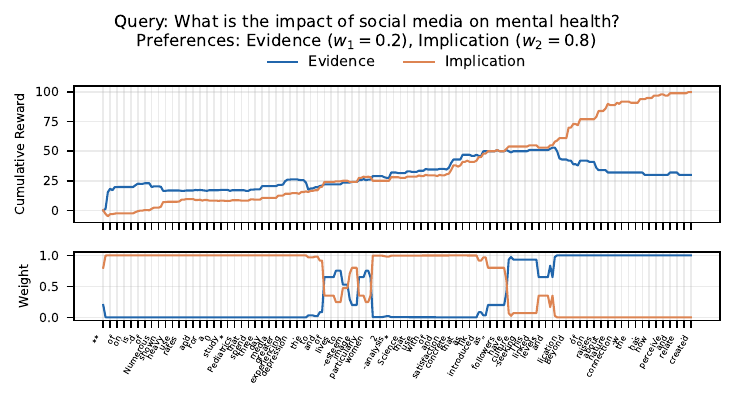}
    \caption{Rewards and weights dynamics for evidence ($w_1=0.2$) and implication ($w_2=0.8$).}
    \label{fig:evidence_vs_implication_0.2-0.8}
\end{figure}
\section{Prompt}

\subsection{Prompt for Reward Discovery}\label{sec:appendix_reward_prompt}

For reward discovery, we instantiate the preference-specific prompt $c_k$ in \Cref{eq:reward_discovery} by conditioning the backbone LLM on a single objective description while keeping the user query and partially generated response fixed. For a query $x$, objective description $c_k$, and decoding prefix $y_{<i}$, we compute $\pi_{\text{base}}(y_i \mid x, y_{<i}, c_k)$ with the following template:

\begin{lstlisting}
You are an AI assistant to provide responses to user questions.

Your responses should follow the following preferences:

[OBJECTIVE DESCRIPTION c_k]

User query:
[QUERY x]

\end{lstlisting}

In practice, \texttt{[OBJECTIVE DESCRIPTION $c_k$]} is instantiated with a natural-language preference from the target preference. The base distribution $\pi_{\text{base}}(y_i \mid x, y_{<i})$ is computed from the same query and prefix without the objective-specific conditioning text.

\subsection{Prompt for Prompt-based Multi-objective Alignment}\label{sec:appendix_alignment_prompt}

For prompt-based multi-objective alignment baselines, as well as for the concatenated prompt $\gC$ used as the stabilizing anchor in \Cref{eq:fused_logits}, we combine all target preference descriptions into a single instruction prompt. For \method, this anchor prompt is constructed once from the preference descriptions and the user's initial weights $\gW^{\mathrm{init}}$, and remains fixed throughout decoding. It provides a shared multi-objective instruction context for the merged distribution. The dynamically optimized weights $\gW^*$ affect distribution fusion rather than the prompt text, allowing the KV cache for this prompt and the generated prefix to be reused and extended across decoding steps. This cache reuse avoids recomputing an altered prompt; it does not eliminate the forward passes needed to generate each new token. When explicit user weights are available, we expose them in natural language so that the model is asked to balance the objectives according to the desired trade-off. The template is as follows.

\begin{lstlisting}
You are an AI assistant to provide responses to user questions.

Your response should follow the multiple principles listed below. Each principle is tagged with a weight indicating its relative importance (higher weight = higher priority). When generating your response, attend to each principle proportionally to its weight and trade off between them accordingly.

1. (weight: [w_1]) [PREFERENCE 1 DESCRIPTION]
2. (weight: [w_2]) [PREFERENCE 2 DESCRIPTION]
...
K. (weight: [w_K]) [PREFERENCE K DESCRIPTION]

User query:
[QUERY]
\end{lstlisting}

\subsection{Prompt for LLM-as-judge Evaluation}\label{sec:appendix_judge_prompt}

For evaluation on the Multifaceted benchmark, we use GPT-4o and claude-sonnet-4-5 as an LLM-as-judge to score each generated response on each assigned preference dimension. The judge receives the user query, one target preference, the corresponding rubric from the dataset, and the model response, and then produces a score $s_k \in [1,5]$. The prompt template is shown below.

\begin{lstlisting}
You are an expert evaluator. You will be given a user's instruction and an AI assistant's response. You will also be given a specific evaluation criterion with detailed score descriptions.

Please evaluate the response based ONLY on the given criterion.

## User Instruction
[Query]

## AI Assistant's Response
[RESPONSE]

## Target Preference
[PREFERENCE NAME OR DESCRIPTION]

## Score Descriptions
- Score 1: {RUBRIC FOR SCORE 1}
- Score 2: {RUBRIC FOR SCORE 2}
- Score 3: {RUBRIC FOR SCORE 3}
- Score 4: {RUBRIC FOR SCORE 4}
- Score 5: {RUBRIC FOR SCORE 5}

Based on the criterion above, assign a score from 1 to 5. Provide a brief justification, then state your final score.

Output your final score in the format: [[score]]
\end{lstlisting}

We run the above prompt independently for each preference assigned to the query. The returned scores are then aggregated into the metrics reported in \Cref{sec:experimental_setup}, including AMR, APS, and Worst.

\section{Limitation}\label{sec:appendix_limitation}
The effectiveness of \method is limited by the backbone's ability to understand and follow the specified preferences. When an objective is poorly represented by the backbone, the prompt-derived surrogate rewards may be unreliable and the resulting alignment may be suboptimal. Our surrogate reward validation is limited to comparisons with LLM judges; it does not establish token-level calibration or agreement with human judgments. Additionally, \method assumes that the objectives can be expressed in natural language. The multiple prompt-conditioned forward passes and online distribution updates also introduce inference overhead compared to standard decoding. A promising future direction is to apply a lightweight SFT/DPO warm-up before \method to strengthen the backbone's preference-following ability, while recognizing that this would introduce additional training and its associated cost.

\newpage
\section*{NeurIPS Paper Checklist}

The checklist is designed to encourage best practices for responsible machine learning research, addressing issues of reproducibility, transparency, research ethics, and societal impact. Do not remove the checklist: {\bf The papers not including the checklist will be desk rejected.} The checklist should follow the references and follow the (optional) supplemental material.  The checklist does NOT count towards the page
limit. 

Please read the checklist guidelines carefully for information on how to answer these questions. For each question in the checklist:
\begin{itemize}
    \item You should answer \answerYes{}, \answerNo{}, or \answerNA{}.
    \item \answerNA{} means either that the question is Not Applicable for that particular paper or the relevant information is Not Available.
    \item Please provide a short (1--2 sentence) justification right after your answer (even for \answerNA). 
\end{itemize}

{\bf The checklist answers are an integral part of your paper submission.} They are visible to the reviewers, area chairs, senior area chairs, and ethics reviewers. You will also be asked to include it (after eventual revisions) with the final version of your paper, and its final version will be published with the paper.

The reviewers of your paper will be asked to use the checklist as one of the factors in their evaluation. While \answerYes{} is generally preferable to \answerNo{}, it is perfectly acceptable to answer \answerNo{} provided a proper justification is given (e.g., error bars are not reported because it would be too computationally expensive'' or ``we were unable to find the license for the dataset we used''). In general, answering \answerNo{} or \answerNA{} is not grounds for rejection. While the questions are phrased in a binary way, we acknowledge that the true answer is often more nuanced, so please just use your best judgment and write a justification to elaborate. All supporting evidence can appear either in the main paper or the supplemental material, provided in appendix. If you answer \answerYes{} to a question, in the justification please point to the section(s) where related material for the question can be found.

IMPORTANT, please:
\begin{itemize}
    \item {\bf Delete this instruction block, but keep the section heading ``NeurIPS Paper Checklist"},
    \item  {\bf Keep the checklist subsection headings, questions/answers and guidelines below.}
    \item {\bf Do not modify the questions and only use the provided macros for your answers}.
\end{itemize}


\begin{enumerate}

\item {\bf Claims}
    \item[] Question: Do the main claims made in the abstract and introduction accurately reflect the paper's contributions and scope?
    \item[] Answer: \answerYes{} 
    \item[] Justification: The main claims of the paper are stated in the abstract and introduction, and they accurately reflect the contributions and scope of the paper. 
    \item[] Guidelines:
    \begin{itemize}
        \item The answer \answerNA{} means that the abstract and introduction do not include the claims made in the paper.
        \item The abstract and/or introduction should clearly state the claims made, including the contributions made in the paper and important assumptions and limitations. A \answerNo{} or \answerNA{} answer to this question will not be perceived well by the reviewers. 
        \item The claims made should match theoretical and experimental results, and reflect how much the results can be expected to generalize to other settings. 
        \item It is fine to include aspirational goals as motivation as long as it is clear that these goals are not attained by the paper. 
    \end{itemize}

\item {\bf Limitations}
    \item[] Question: Does the paper discuss the limitations of the work performed by the authors?
    \item[] Answer: \answerYes{} 
    \item[] Justification: The paper includes a discussion of the limitations in \Cref{sec:appendix_limitation}.
    \item[] Guidelines:
    \begin{itemize}
        \item The answer \answerNA{} means that the paper has no limitation while the answer \answerNo{} means that the paper has limitations, but those are not discussed in the paper. 
        \item The authors are encouraged to create a separate ``Limitations'' section in their paper.
        \item The paper should point out any strong assumptions and how robust the results are to violations of these assumptions (e.g., independence assumptions, noiseless settings, model well-specification, asymptotic approximations only holding locally). The authors should reflect on how these assumptions might be violated in practice and what the implications would be.
        \item The authors should reflect on the scope of the claims made, e.g., if the approach was only tested on a few datasets or with a few runs. In general, empirical results often depend on implicit assumptions, which should be articulated.
        \item The authors should reflect on the factors that influence the performance of the approach. For example, a facial recognition algorithm may perform poorly when image resolution is low or images are taken in low lighting. Or a speech-to-text system might not be used reliably to provide closed captions for online lectures because it fails to handle technical jargon.
        \item The authors should discuss the computational efficiency of the proposed algorithms and how they scale with dataset size.
        \item If applicable, the authors should discuss possible limitations of their approach to address problems of privacy and fairness.
        \item While the authors might fear that complete honesty about limitations might be used by reviewers as grounds for rejection, a worse outcome might be that reviewers discover limitations that aren't acknowledged in the paper. The authors should use their best judgment and recognize that individual actions in favor of transparency play an important role in developing norms that preserve the integrity of the community. Reviewers will be specifically instructed to not penalize honesty concerning limitations.
    \end{itemize}

\item {\bf Theory assumptions and proofs}
    \item[] Question: For each theoretical result, does the paper provide the full set of assumptions and a complete (and correct) proof?
    \item[] Answer:  \answerYes{} 
    \item[] Justification: The paper provides the detailed discussion of the theoretical analysis in \Cref{sec:appendix_proofs}.
    \item[] Guidelines:
    \begin{itemize}
        \item The answer \answerNA{} means that the paper does not include theoretical results. 
        \item All the theorems, formulas, and proofs in the paper should be numbered and cross-referenced.
        \item All assumptions should be clearly stated or referenced in the statement of any theorems.
        \item The proofs can either appear in the main paper or the supplemental material, but if they appear in the supplemental material, the authors are encouraged to provide a short proof sketch to provide intuition. 
        \item Inversely, any informal proof provided in the core of the paper should be complemented by formal proofs provided in appendix or supplemental material.
        \item Theorems and Lemmas that the proof relies upon should be properly referenced. 
    \end{itemize}

    \item {\bf Experimental result reproducibility}
    \item[] Question: Does the paper fully disclose all the information needed to reproduce the main experimental results of the paper to the extent that it affects the main claims and/or conclusions of the paper (regardless of whether the code and data are provided or not)?
    \item[] Answer:  \answerYes{} 
    \item[] Justification: The paper provides the detailed discussion of the experimental setup in \Cref{sec:experimental_setup}.
    \item[] Guidelines:
    \begin{itemize}
        \item The answer \answerNA{} means that the paper does not include experiments.
        \item If the paper includes experiments, a \answerNo{} answer to this question will not be perceived well by the reviewers: Making the paper reproducible is important, regardless of whether the code and data are provided or not.
        \item If the contribution is a dataset and\slash or model, the authors should describe the steps taken to make their results reproducible or verifiable. 
        \item Depending on the contribution, reproducibility can be accomplished in various ways. For example, if the contribution is a novel architecture, describing the architecture fully might suffice, or if the contribution is a specific model and empirical evaluation, it may be necessary to either make it possible for others to replicate the model with the same dataset, or provide access to the model. In general. releasing code and data is often one good way to accomplish this, but reproducibility can also be provided via detailed instructions for how to replicate the results, access to a hosted model (e.g., in the case of a large language model), releasing of a model checkpoint, or other means that are appropriate to the research performed.
        \item While NeurIPS does not require releasing code, the conference does require all submissions to provide some reasonable avenue for reproducibility, which may depend on the nature of the contribution. For example
        \begin{enumerate}
            \item If the contribution is primarily a new algorithm, the paper should make it clear how to reproduce that algorithm.
            \item If the contribution is primarily a new model architecture, the paper should describe the architecture clearly and fully.
            \item If the contribution is a new model (e.g., a large language model), then there should either be a way to access this model for reproducing the results or a way to reproduce the model (e.g., with an open-source dataset or instructions for how to construct the dataset).
            \item We recognize that reproducibility may be tricky in some cases, in which case authors are welcome to describe the particular way they provide for reproducibility. In the case of closed-source models, it may be that access to the model is limited in some way (e.g., to registered users), but it should be possible for other researchers to have some path to reproducing or verifying the results.
        \end{enumerate}
    \end{itemize}

\item {\bf Open access to data and code}
    \item[] Question: Does the paper provide open access to the data and code, with sufficient instructions to faithfully reproduce the main experimental results, as described in supplemental material?
    \item[] Answer: \answerYes{}
    \item[] Justification: The code and data are available at \url{https://anonymous.4open.science/r/MATO-B745/}.
    \item[] Guidelines:
    \begin{itemize}
        \item The answer \answerNA{} means that paper does not include experiments requiring code.
        \item Please see the NeurIPS code and data submission guidelines (\url{https://neurips.cc/public/guides/CodeSubmissionPolicy}) for more details.
        \item While we encourage the release of code and data, we understand that this might not be possible, so \answerNo{} is an acceptable answer. Papers cannot be rejected simply for not including code, unless this is central to the contribution (e.g., for a new open-source benchmark).
        \item The instructions should contain the exact command and environment needed to run to reproduce the results. See the NeurIPS code and data submission guidelines (\url{https://neurips.cc/public/guides/CodeSubmissionPolicy}) for more details.
        \item The authors should provide instructions on data access and preparation, including how to access the raw data, preprocessed data, intermediate data, and generated data, etc.
        \item The authors should provide scripts to reproduce all experimental results for the new proposed method and baselines. If only a subset of experiments are reproducible, they should state which ones are omitted from the script and why.
        \item At submission time, to preserve anonymity, the authors should release anonymized versions (if applicable).
        \item Providing as much information as possible in supplemental material (appended to the paper) is recommended, but including URLs to data and code is permitted.
    \end{itemize}

\item {\bf Experimental setting/details}
    \item[] Question: Does the paper specify all the training and test details (e.g., data splits, hyperparameters, how they were chosen, type of optimizer) necessary to understand the results?
    \item[] Answer:  \answerYes{} 
    \item[] Justification: The paper provides the detailed discussion of the experimental setup in \Cref{sec:experimental_setup}.
    \item[] Guidelines:
    \begin{itemize}
        \item The answer \answerNA{} means that the paper does not include experiments.
        \item The experimental setting should be presented in the core of the paper to a level of detail that is necessary to appreciate the results and make sense of them.
        \item The full details can be provided either with the code, in appendix, or as supplemental material.
    \end{itemize}

\item {\bf Experiment statistical significance}
    \item[] Question: Does the paper report error bars suitably and correctly defined or other appropriate information about the statistical significance of the experiments?
    \item[] Answer: \answerNA{} 
    \item[] Justification: The random seed is fixed for all experiments, and the results are deterministic.
    \item[] Guidelines:
    \begin{itemize}
        \item The answer \answerNA{} means that the paper does not include experiments.
        \item The authors should answer \answerYes{} if the results are accompanied by error bars, confidence intervals, or statistical significance tests, at least for the experiments that support the main claims of the paper.
        \item The factors of variability that the error bars are capturing should be clearly stated (for example, train/test split, initialization, random drawing of some parameter, or overall run with given experimental conditions).
        \item The method for calculating the error bars should be explained (closed form formula, call to a library function, bootstrap, etc.)
        \item The assumptions made should be given (e.g., Normally distributed errors).
        \item It should be clear whether the error bar is the standard deviation or the standard error of the mean.
        \item It is OK to report 1-sigma error bars, but one should state it. The authors should preferably report a 2-sigma error bar than state that they have a 96\% CI, if the hypothesis of Normality of errors is not verified.
        \item For asymmetric distributions, the authors should be careful not to show in tables or figures symmetric error bars that would yield results that are out of range (e.g., negative error rates).
        \item If error bars are reported in tables or plots, the authors should explain in the text how they were calculated and reference the corresponding figures or tables in the text.
    \end{itemize}

\item {\bf Experiments compute resources}
    \item[] Question: For each experiment, does the paper provide sufficient information on the computer resources (type of compute workers, memory, time of execution) needed to reproduce the experiments?
    \item[] Answer: \answerYes{} 
    \item[] Justification: The paper provides the detailed discussion of the compute resources in \Cref{sec:appendix_implementation}.
    \item[] Guidelines:
    \begin{itemize}
        \item The answer \answerNA{} means that the paper does not include experiments.
        \item The paper should indicate the type of compute workers CPU or GPU, internal cluster, or cloud provider, including relevant memory and storage.
        \item The paper should provide the amount of compute required for each of the individual experimental runs as well as estimate the total compute. 
        \item The paper should disclose whether the full research project required more compute than the experiments reported in the paper (e.g., preliminary or failed experiments that didn't make it into the paper). 
    \end{itemize}
    
\item {\bf Code of ethics}
    \item[] Question: Does the research conducted in the paper conform, in every respect, with the NeurIPS Code of Ethics \url{https://neurips.cc/public/EthicsGuidelines}?
    \item[] Answer: \answerYes{} 
    \item[] Justification: The paper adheres to the NeurIPS Code of Ethics.
    \item[] Guidelines:
    \begin{itemize}
        \item The answer \answerNA{} means that the authors have not reviewed the NeurIPS Code of Ethics.
        \item If the authors answer \answerNo, they should explain the special circumstances that require a deviation from the Code of Ethics.
        \item The authors should make sure to preserve anonymity (e.g., if there is a special consideration due to laws or regulations in their jurisdiction).
    \end{itemize}

\item {\bf Broader impacts}
    \item[] Question: Does the paper discuss both potential positive societal impacts and negative societal impacts of the work performed?
    \item[] Answer: \answerNA{} 
    \item[] Justification: The paper does not include a discussion of societal impacts because the work is purely theoretical and does not have any foreseeable societal impacts.
    \item[] Guidelines:
    \begin{itemize}
        \item The answer \answerNA{} means that there is no societal impact of the work performed.
        \item If the authors answer \answerNA{} or \answerNo, they should explain why their work has no societal impact or why the paper does not address societal impact.
        \item Examples of negative societal impacts include potential malicious or unintended uses (e.g., disinformation, generating fake profiles, surveillance), fairness considerations (e.g., deployment of technologies that could make decisions that unfairly impact specific groups), privacy considerations, and security considerations.
        \item The conference expects that many papers will be foundational research and not tied to particular applications, let alone deployments. However, if there is a direct path to any negative applications, the authors should point it out. For example, it is legitimate to point out that an improvement in the quality of generative models could be used to generate Deepfakes for disinformation. On the other hand, it is not needed to point out that a generic algorithm for optimizing neural networks could enable people to train models that generate Deepfakes faster.
        \item The authors should consider possible harms that could arise when the technology is being used as intended and functioning correctly, harms that could arise when the technology is being used as intended but gives incorrect results, and harms following from (intentional or unintentional) misuse of the technology.
        \item If there are negative societal impacts, the authors could also discuss possible mitigation strategies (e.g., gated release of models, providing defenses in addition to attacks, mechanisms for monitoring misuse, mechanisms to monitor how a system learns from feedback over time, improving the efficiency and accessibility of ML).
    \end{itemize}
    
\item {\bf Safeguards}
    \item[] Question: Does the paper describe safeguards that have been put in place for responsible release of data or models that have a high risk for misuse (e.g., pre-trained language models, image generators, or scraped datasets)?
    \item[] Answer: \answerNA{} 
    \item[] Justification: The paper does not include a discussion of safeguards because the work is purely theoretical and does not involve the release of data or models.
    \item[] Guidelines:
    \begin{itemize}
        \item The answer \answerNA{} means that the paper poses no such risks.
        \item Released models that have a high risk for misuse or dual-use should be released with necessary safeguards to allow for controlled use of the model, for example by requiring that users adhere to usage guidelines or restrictions to access the model or implementing safety filters. 
        \item Datasets that have been scraped from the Internet could pose safety risks. The authors should describe how they avoided releasing unsafe images.
        \item We recognize that providing effective safeguards is challenging, and many papers do not require this, but we encourage authors to take this into account and make a best faith effort.
    \end{itemize}

\item {\bf Licenses for existing assets}
    \item[] Question: Are the creators or original owners of assets (e.g., code, data, models), used in the paper, properly credited and are the license and terms of use explicitly mentioned and properly respected?
    \item[] Answer: \answerYes{} 
    \item[] Justification: The paper adheres to the licenses of the assets used in the paper, and the creators or original owners of the assets are properly credited.
    \item[] Guidelines:
    \begin{itemize}
        \item The answer \answerNA{} means that the paper does not use existing assets.
        \item The authors should cite the original paper that produced the code package or dataset.
        \item The authors should state which version of the asset is used and, if possible, include a URL.
        \item The name of the license (e.g., CC-BY 4.0) should be included for each asset.
        \item For scraped data from a particular source (e.g., website), the copyright and terms of service of that source should be provided.
        \item If assets are released, the license, copyright information, and terms of use in the package should be provided. For popular datasets, \url{paperswithcode.com/datasets} has curated licenses for some datasets. Their licensing guide can help determine the license of a dataset.
        \item For existing datasets that are re-packaged, both the original license and the license of the derived asset (if it has changed) should be provided.
        \item If this information is not available online, the authors are encouraged to reach out to the asset's creators.
    \end{itemize}

\item {\bf New assets}
    \item[] Question: Are new assets introduced in the paper well documented and is the documentation provided alongside the assets?
    \item[] Answer: \answerNo{}
    \item[] Justification: There are no new assets introduced in the paper.
    \item[] Guidelines:
    \begin{itemize}
        \item The answer \answerNA{} means that the paper does not release new assets.
        \item Researchers should communicate the details of the dataset\slash code\slash model as part of their submissions via structured templates. This includes details about training, license, limitations, etc. 
        \item The paper should discuss whether and how consent was obtained from people whose asset is used.
        \item At submission time, remember to anonymize your assets (if applicable). You can either create an anonymized URL or include an anonymized zip file.
    \end{itemize}

\item {\bf Crowdsourcing and research with human subjects}
    \item[] Question: For crowdsourcing experiments and research with human subjects, does the paper include the full text of instructions given to participants and screenshots, if applicable, as well as details about compensation (if any)? 
    \item[] Answer: \answerNA{} 
    \item[] Justification: There are no crowdsourcing experiments or research with human subjects in the paper.
    \item[] Guidelines:
    \begin{itemize}
        \item The answer \answerNA{} means that the paper does not involve crowdsourcing nor research with human subjects.
        \item Including this information in the supplemental material is fine, but if the main contribution of the paper involves human subjects, then as much detail as possible should be included in the main paper. 
        \item According to the NeurIPS Code of Ethics, workers involved in data collection, curation, or other labor should be paid at least the minimum wage in the country of the data collector. 
    \end{itemize}

\item {\bf Institutional review board (IRB) approvals or equivalent for research with human subjects}
    \item[] Question: Does the paper describe potential risks incurred by study participants, whether such risks were disclosed to the subjects, and whether Institutional Review Board (IRB) approvals (or an equivalent approval/review based on the requirements of your country or institution) were obtained?
    \item[] Answer: \answerNA{} 
    \item[] Justification: There are no research with human subjects in the paper.
    \item[] Guidelines:
    \begin{itemize}
        \item The answer \answerNA{} means that the paper does not involve crowdsourcing nor research with human subjects.
        \item Depending on the country in which research is conducted, IRB approval (or equivalent) may be required for any human subjects research. If you obtained IRB approval, you should clearly state this in the paper. 
        \item We recognize that the procedures for this may vary significantly between institutions and locations, and we expect authors to adhere to the NeurIPS Code of Ethics and the guidelines for their institution. 
        \item For initial submissions, do not include any information that would break anonymity (if applicable), such as the institution conducting the review.
    \end{itemize}

\item {\bf Declaration of LLM usage}
    \item[] Question: Does the paper describe the usage of LLMs if it is an important, original, or non-standard component of the core methods in this research? Note that if the LLM is used only for writing, editing, or formatting purposes and does \emph{not} impact the core methodology, scientific rigor, or originality of the research, declaration is not required.
    \item[] Answer: \answerNA{} 
    \item[] Justification: The core method development in this research does not involve LLMs as any important, original, or non-standard components.
    \item[] Guidelines:
    \begin{itemize}
        \item The answer \answerNA{} means that the core method development in this research does not involve LLMs as any important, original, or non-standard components.
        \item Please refer to our LLM policy in the NeurIPS handbook for what should or should not be described.
    \end{itemize}

\end{enumerate}

\end{document}